\documentclass{article}

\usepackage[table]{xcolor}
\usepackage{microtype}
\usepackage{subfigure}
\usepackage{booktabs} 
\usepackage{enumerate}
\usepackage{hyperref}

\usepackage[accepted]{icml2025}

\usepackage{amsmath,amsfonts,bm}

\def\1{\bm{1}}

\DeclareMathAlphabet{\mathsfit}{\encodingdefault}{\sfdefault}{m}{sl}
\SetMathAlphabet{\mathsfit}{bold}{\encodingdefault}{\sfdefault}{bx}{n}

\usepackage{ulem}
\usepackage{bbm}
\usepackage{booktabs}
\usepackage{multirow}
\usepackage{graphicx}
\usepackage{subcaption}
\usepackage{makecell}
\usepackage{booktabs}
\usepackage{array}
\usepackage{url}
\usepackage{algorithm}
\usepackage{algorithmic}
\usepackage{dsfont}
\usepackage{array}

\usepackage{amsmath}
\usepackage{amssymb}
\usepackage{mathtools}
\usepackage{amsthm}
\usepackage{pifont}

\usepackage[capitalize,noabbrev]{cleveref}

\newcommand{\cH}{\mathcal{H}}
\newcommand{\cI}{\mathcal{I}}

\newcommand{\cL}{\mathcal{L}}
\newcommand{\cM}{\mathcal{M}}
\newcommand{\cN}{\mathcal{N}}

\newcommand{\RR}{\mathbb{R}}

\usepackage[english]{babel}

\definecolor{lightblue}{RGB}{204, 255, 255}

\theoremstyle{plain}
\newtheorem{theorem}{Theorem}[section]
\newtheorem{proposition}[theorem]{Proposition}
\newtheorem{lemma}[theorem]{Lemma}

\theoremstyle{definition}

\newtheorem{assumption}[theorem]{Assumption}

\newtheorem{observation}[theorem]{Observation}
\theoremstyle{remark}

\usepackage[textsize=tiny]{todonotes}

\begin{document}

\twocolumn[
\icmltitle{Memory-Efficient LLM Training by Various-Grained \\ Low-Rank Projection of Gradients}





\icmlsetsymbol{equal}{*}

\begin{icmlauthorlist}
\icmlauthor{Yezhen Wang}{equal,yyy}
\icmlauthor{Zhouhao Yang}{equal,zzz}
\icmlauthor{Brian K. Chen}{yyy}
\icmlauthor{Fanyi Pu}{xxx}
\icmlauthor{Bo Li}{xxx}
\icmlauthor{Tianyu Gao}{yyy}
\icmlauthor{Kenji Kawaguchi}{yyy}
\end{icmlauthorlist}

\icmlaffiliation{yyy}{National University of Singapore}
\icmlaffiliation{zzz}{Johns Hopkins University}
\icmlaffiliation{xxx}{S-Lab, Nanyang Technological University}

\icmlcorrespondingauthor{Yezhen Wang}{yezhenw@comp.nus.edu.sg}

\icmlkeywords{Machine Learning, ICML}

\vskip 0.3in
]



\printAffiliationsAndNotice{\icmlEqualContribution} 

\begin{abstract}
Building upon the success of low-rank adapter (LoRA), low-rank gradient projection (LoRP) has emerged as a promising solution for memory-efficient fine-tuning. However, existing LoRP methods typically treat each row of the gradient matrix as the default projection unit, leaving the role of \textit{projection granularity} underexplored. In this work, we propose a novel framework, \textbf{VLoRP}, that extends low-rank gradient projection by introducing an additional degree of freedom for controlling the trade-off between memory efficiency and performance, beyond the rank hyper-parameter. Through this framework, we systematically explore the impact of projection granularity, demonstrating that finer-grained projections lead to enhanced stability and efficiency even under a fixed memory budget. Regarding the optimization for VLoRP, we present \textbf{ProjFactor}, an adaptive memory-efficient optimizer, that significantly reduces memory requirement while ensuring competitive performance, even in the presence of gradient accumulation. Additionally, we provide a theoretical analysis of VLoRP, demonstrating the descent and convergence of its optimization trajectory under both SGD and ProjFactor. Extensive experiments are conducted to validate our findings, covering tasks such as commonsense reasoning, MMLU, and GSM8K.
\end{abstract}


\section{Introduction}
Large Language Models (LLMs) have demonstrated remarkable capabilities across various domains~\cite{achiam2023gpt, team2023gemini, touvron2023llama, dubey2024llama, li2024llava, guo2025deepseek}. However, these state-of-the-art models impose substantial memory requirements, making them challenging to finetune in practical settings. For example, fine-tuning an LLM with 7 billion parameters in the bfloat16 format~\cite{dean2012large, abadi2016tensorflow} requires approximately 14 GB of memory for the parameters alone. The gradients demand a similar amount of memory,\footnote{Techniques like layer-wise updates can mitigate this overhead, but gradient accumulation still necessitates storing gradients in practice.} and, with Adam/AdamW~\citep{kingma2014adam, adamw} serving as the de facto optimizer, an additional 28 GB is needed for optimizer states. Including the activations required for backpropagation, the total memory footprint exceeds 60 GB, rendering such training prohibitively costly for most users.

To mitigate this training overhead, various parameter-efficient fine-tuning (PeFT) techniques \cite{Houlsby2019, Razdaibiedina2023} have been developed. Among the most notable are Low-Rank Adapters (LoRA) \citep{lora} and its variants \citep{QLoRA, DoRA, LoRA+, LoRA-GA}, which decompose matrix-shaped parameters into two low-rank matrices to enable more memory-efficient training. Recently, a novel family of methods, Low-Rank Gradient Projection (LoRP) \citep{FLoRA, Galore, fira, appollo}, has emerged. LoRP methods also leverage low-rank properties during LLM training, but rather than operating at the parameter level, they focus on exploiting the low-rank structure within the gradient. Concretely, LoRP methods achieve memory reduction by projecting the gradient matrix $G\in\RR^{n\times m}$ into a low-dimensional subspace spanned by the columns of a projection matrix $P\in\RR^{m\times r}$ with $r\ll m$. When required for parameter updates, the gradient is approximated as $(GP)P^\top$ by projecting the stored low-dimensional gradient $GP$ back to the original space. In general, LoRP can offer better memory efficiency than LoRA, as it only requires the storage of gradient information in a single low-dimensional subspace, whereas LoRA necessitates two. 

An interesting observation is that LoRP can also be interpreted from the perspective of stochastic approximation, specifically in the form of the forward gradient method~\cite{baydin22}: when $P$ is a random matrix sampled from Gaussian distribution $\mathcal{N}(0, \frac{1}{r})$, LoRP can be viewed as an unbiased stochastic approximation of the original gradient, applied row-wisely (elaborated in Sec.\ref{sec:3}):
\begin{equation}
     G_{i, :} \approx G_{i, :}PP^{\top} = \frac{1}{r}\sum_{j=1}^r \left(G_{i, :}\cdot v_j \right)v^{\top}_j,
    \label{eq:fwd_grad}
\end{equation}
where $i = 1, \dots, n$, $v_j\sim \mathcal{N}\left(0, I_m\right)$ is the $j$-th column of $\sqrt{r}P$, and $G_{i, :}$ is the $i$-th row of $G$, representing a segment of the gradient being approximated. In \eqref{eq:fwd_grad}, $r$ serves a dual role: it defines the dimensionality of the projected space in LoRP and also determines the number of samples used for stochastic approximation, akin to the query count in Monte Carlo estimation. 
From the lens of stochastic approximation, the effectiveness of estimation is highly influenced simultaneously by the target dimensionality (\textit{i.e.}, the size of $G_{i,:}$) and the number of samples $r$~\cite{berahas2022theoretical,FG2U}. If $r$ increases, more independent samples can reduce variance in \eqref{eq:fwd_grad}. Alternatively, keeping 
$r$ fixed while decreasing the dimensionality of 
$G_{i,:}$ improves accuracy by reducing the number of dimensions to estimate. This insight naturally leads to a new degree of freedom  in LoRP methods: \textbf{rather than fixing the dimensionality of 
$G_{i,:}$, we can adjust it together with 
$r$ to balance performance and memory efficiency.}

Building on the above observation, we introduce the concept \textbf{Projection Granularity} defined as the length of $G$'s rows, which essentially represents the size of the fundamental projection unit in LoRP. In response,  we propose \textbf{VLoRP} (\textbf{V}arious-Grained \textbf{Lo}w-\textbf{R}ank \textbf{P}rojection of Gradients), a novel framework that simultaneously adjusts both the Projection Granularity and the rank \( r \). By systematically exploring different configurations of granularity and rank in VLoRP, we find that, under a fixed memory budget, selecting a finer granularity, even with a smaller rank, consistently leads to superior performance. Additionally, by explicitly analyzing the mean and variance of gradient estimation, we establish that VLoRP achieves an \( O(1/T) \) convergence rate when paired with the SGD optimizer.


On top of it, given that Adam-based optimizers~\cite{kingma2014adam, adafactor, adamw} are the de facto standard for training modern LLMs, we also explore the potential adaptive training schemes for the VLoRP framework. Specifically, we investigate two distinct optimization schemes: the \textbf{Subspace Scheme (SS)} and the \textbf{Original Space Scheme (OS)}. In \textbf{SS}, optimization states are maintained within the low-dimensional subspace.
In contrast, \textbf{OS} projects the low-dimensional gradients back first before storing and updating the optimization states within the original space. While both schemes only access the rank-\(r\) alternative of the original gradients, our results indicate that \textbf{OS} typically outperforms \textbf{SS}. However, since \textbf{OS} stores the optimizer states in the original full-rank space, its memory requirements can become prohibitive. To address this issue, we propose \textbf{ProjFactor}, a memory-efficient variant of \textbf{OS} that significantly reduces the memory usage for storing optimization states while maintaining comparable performance. Finally, we adopt Hamiltonian descent methods~\cite{maddison2018hamiltonian, chen2023lion, liang2024memory} to show that, despite several approximations made as a compromise of memory efficiency, the Lyapunov function regarding ProjFactor can monotonically decrease with time, ensuring that the optimizer converges to a local optimum.

Our main contributions are as follows:
(i) We introduce VLoRP, a method that facilitates the simultaneous adjustment of Projection Granularity and rank, thereby offering a more nuanced control over the trade-off between memory efficiency and performance. Besides, we argue that finer granularity is more significant than larger rank and validate it with comprehensive experiments.
(ii) We investigate two Adam-based optimization schemes for the VLoRP framework and propose a novel adaptive optimization algorithm, named ProjFactor. This approach significantly reduces memory consumption while maintaining competitive performance.
(iii) We present theoretical proof for the convergence of VLoRP with the SGD optimizer, achieving an $O(1/T)$ convergence rate. Additionally, we provide a rigorous theoretical guarantee for VLoRP under ProjFactor based on the framework of Hamiltonian descent.

\section{Background}
\textbf{Low-Rank Gradient Projection (LoRP)} \quad
Recently, LoRP algorithms~\cite{FLoRA, Galore, fira, appollo, adamem} have become prominent PeFT methods. These methods leverage the observation that the gradients in high-dimensional parameter spaces often lie in a low-dimensional manifold. For instance, FLoRA~\citep{FLoRA} demonstrates that LoRA can be approximated by down-projecting gradients into a low-rank subspace and then up-projecting back using the same random projection matrix. Galore~\citep{Galore} follows a similar spirit but derives the projection matrix through the Singular Value Decomposition (SVD) of the gradients. \citet{fira} extends Galore by incorporating the residual error between the full-rank gradient and its low-rank approximation, effectively simulating full-rank updates. APOLLO~\citep{appollo} approximates
channel-wise learning-rate scaling through an auxiliary low-rank optimizer state derived from random projections. However, existing LoRP approaches typically treat each row of the gradient matrix as the default projection unit, leaving the influence of varying Projection Granularity unexplored.
\begin{figure*}[t]
    \centering
    \includegraphics[width=1.0\linewidth]{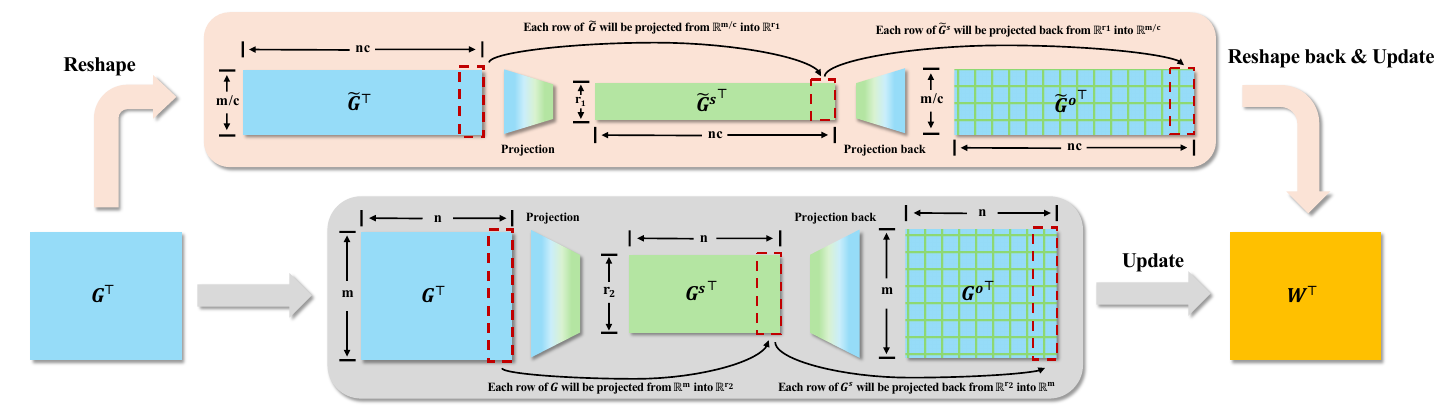} 
    \caption{
    Overview of VLoRP versus standard LoRP.
    \textit{Bottom (gray)} : In ordinary LoRP, the gradient matrix $G$ is directly projected row‐by‐row from $\mathbb{R}^{n\times m}$ into $\mathbb{R}^{n\times r}$ and stored as $G^{s}$.
    \textit{Top (pale beige)}: In contrast, VLoRP reshapes the original gradient matrix $G$ first to adjust the granularity of projection (from $\mathbb{R}^{n\times m}$ to $\mathbb{R}^{nc \times\frac{m}{c}}$), and during the update, the project-backed gradient $\tilde{G}^o$ would be reshaped back into $\mathbb{R}^{n\times m}$ to update the parameter $W\in \mathbb{R}^{n\times m}$.
    }
    \label{fig:VLoRP}
\end{figure*}

\textbf{Forward Gradient (FG)} \quad 
Forward Gradient (FG) is a widely used stochastic approximation technique for gradient estimation, particularly in scenarios where the direct computation of gradients is infeasible or expensive~\citep{Wengert64, SilverGDHH22, baydin22, RenKLH23}. Instead of computing exact gradients, FG estimates them by taking directional derivatives along multiple random perturbation directions.
Formally, for a differentiable model with parameters 
$\theta\in\RR^d$, FG approximates the gradient of the objective function 
$\cL$ as: 
\begin{equation} 
\begin{aligned} 
\label{eq:forward_gradient} 
\nabla_{\theta} \mathcal{L}(\theta) \approx \frac{1}{N} \sum_{i=1}^{N} \left(\nabla_\theta \mathcal{L}(\theta)^{\top}z_i  \right) z_i, 
\end{aligned} 
\end{equation} 
where \(N\) is the number of random perturbation directions, and $z_i$ denotes the i.i.d. sampled random vector. 

More broadly, FG belongs to the family of stochastic approximation methods~\cite{robbins1951stochastic,nevel1976stochastic, SA92}, which iteratively estimate function properties based on noisy or partial observations. These methods are widely used in optimization and root-finding problems where exact computations are impractical. 

For brevity, additional related work is deferred to \cref{appendix:related_works}, and the notation list is provided in \cref{notations}.

\section{Exploring Various Granularities of Low-Rank Gradient Projection}
\label{sec:3}
In this section, we provide a comprehensive understanding of LoRP methods and introduce VLoRP, which leverages the concept of Projection Granularity to optimize gradient projection. In Section~\ref{sec:Understanding LoRPs through the Lens of Stochastic Approximation}, we reinterpret LoRP through the lens of stochastic approximation. In Section~\ref{sec:Various-Grained Low-Rank Projection of Gradient}, we define Projection Granularity as a new degree of freedom in VLoRP, enabling a subtler control of the trade-off between memory efficiency and performance.
Next, in Section~\ref{sec:Projections with Finer Granularity Fairly Demonstrate Superior Efficacy}, we show that finer Projection Granularity under a fixed ``Memory Budget" improves performance, highlighting its importance over rank. Finally, in Section~\ref{sec:Theoretical Guarantee of VLoRP}, we prove an 
$O(1/T)$ convergence rate for VLoRP with SGD, with theoretical proofs provided in \cref{appendix: theorem_proofs}.
\subsection{Understanding LoRPs through the Lens of Stochastic Approximation}
\label{sec:Understanding LoRPs through the Lens of Stochastic Approximation} 
Given the gradient matrix \( G \in \mathbb{R}^{n \times m} \) and the projection matrix $P\in \mathbb{R}^{m\times r}$, LoRP approaches project the gradient into a low-dimensional subspace spanned by columns of projection matrix $P$, and project it back to the original space when gradient information is required for parameter update:
\begin{equation}
\label{eq:proj_projback}
    G^s := GP, \quad G^{o} := \gamma G^sP^{\top} = \gamma GPP^{\top},
\end{equation}
where we use $G^s$ and $G^o$ to represent the projected gradient in the subspace and the project-backed rank-$r$ approximation of the original gradient respectively, and $\gamma$ denotes the scaling coefficient and usually set as $1$. In practice, only the matrix \( G^s \in \mathbb{R}^{n \times r} \) needs to be stored with  \( r \ll \min\{m, n\} \). Particularly, if each element of $P$ is i.i.d. sampled from Gaussian distribution $\cN(0, \frac{1}{r})$, then $G^o$ can be reformulated into
\begin{equation}
\begin{aligned}
G^o = \frac{1}{r}\sum\limits_{j=1}^{r}G v_j v_j\top 
=\left(
  \begin{array}{c}
    \frac{1}{r}\sum\limits_{j=1}^r \left(G_{1,:}\cdot v_j \right) v^{\top}_j \\
    \vdots \\
    \frac{1}{r}\sum\limits_{j=1}^r \left(G_{n, :}\cdot v_j\right) v_j^\top \\
  \end{array}
\right),
\label{eq:flora2FG}
\end{aligned}
\end{equation}
where $v_j\sim \mathcal{N}\left(0, I_m\right)$ is the $j$-th column vector of $\sqrt{r}P$, and $G_{i, :}\in\mathbb{R}^{1\times m}$ is the $i$-th row of the gradient matrix $G$. 
\eqref{eq:flora2FG} highlights that: (a) \( G^o \) can be represented as an average over \( r \) rank-one estimation, each involving one random direction \( v_j \); (b) Each row \( \frac{1}{r}\sum_{j=1}^r (v_j^\top G_{i, :}^\top) v_j^\top \) aligns with the formulation of the averaged forward gradients as in~\eqref{eq:forward_gradient}. Thus, we have the following observation: 
\begin{observation}\label{obs:lora_rs}
\textit{Equation \eqref{eq:flora2FG} indicates that for each parameter matrix, LoRP can be interpreted as a row-wise application of forward gradient estimation with \( r \) samples. }
\end{observation}

In general, when training an LLM parametrized by \(\theta\) with an objective function \(\mathcal{L}\), the stochastic approximation of the gradient can be applied at different levels. Consider two extreme cases: 
(1) We can flatten the entire parameter set \(\theta\) into a single vector \(\theta' \in \mathbb{R}^D\), where \(D\) is the total number of parameters, and estimate the gradient using a single random vector \(v \in \mathbb{R}^D\). This approach aligns with the spirit of the MeZO algorithm~\cite{mezo}\footnote{MeZO uses the central difference to estimate the JVP in \eqref{eq:forward_gradient}.}; 
(2) Alternatively, the estimation can be applied element-wise: for each 1-dimensional parameter \(\xi \in \theta\), the coordinate partial derivative \(\nabla_{\xi}\mathcal{L}\) can be approximated as \(\frac{1}{r'}\sum_{i=1}^{r'} v_i \nabla_{\xi} \mathcal{L} v_i\), where \(v_i \sim \mathcal{N}(0, 1)\). Concatenating all such approximations yields an estimate of the full gradient matrix \(G\), resembling the coordinate gradient estimation (CGE) approach~\citep{deepzero}.

These examples illustrate that stochastic approximation can be applied to parameters of varying sizes, leading to different methods. Naturally, according to \cref{obs:lora_rs}, the Projection Granularity, defined by the row size of $G$, can also be adjusted in LoRPs. This insight motivates our exploration of low-rank gradient projections with varying granularities.

\subsection{Various-Grained Low-Rank Projection of Gradient}
\label{sec:Various-Grained Low-Rank Projection of Gradient}
In this section, we present the details of our framework, termed \textbf{VLoRP}, which introduces a novel degree of freedom—namely, the Projection Granularity—beyond the conventional hyperparameter rank to balance memory efficiency and training performance in LoRP approaches.

The core idea is straightforward: because gradient projection is typically applied row-wisely, one can reshape the gradient matrix to modify the row size and correspondingly adjust the shape of the projection matrix generated, leading to the adjustment of the Projection Granularity naturally. Concretely, given an LLM, we introduce the \textbf{granularity factor} \( c \), which can reshape any gradient \( G \in \mathbb{R}^{n \times m} \) into \( \tilde{G} \in \mathbb{R}^{nc \times (m/c)} \). Analogous to the rank \( r \), which globally sets the projection rank for all matrix-shaped parameters, \( c \) similarly serves as a global hyperparameter that controls the granularity of projection. Then, we 
formally have\footnote{with a little abuse of notations, $G^o$ in \eqref{eq:proj_projback} is now redefined in \eqref{eq:RDP}, as they both represent the rank-$r$ estimation of $G$.}:
\begin{equation}
\label{eq:RDP}
\begin{aligned}
\tilde{G}^s &:= \underbrace{\operatorname{Reshape}\left( G, [nc, \tfrac{m}{c}] \right)}_{\text{denoted as } \tilde{G}} \tilde{P},\\
G^o &:= \operatorname{Reshape}\Bigl(\underbrace{\tilde{G}^s\tilde{P}^{\top}}_{\text{denoted as }\tilde{G}^o},\,[n, m]\Bigr),
\end{aligned}
\end{equation}
where the projection matrix \( \tilde{P} \) is of shape \( \tfrac{m}{c} \times r \), with each element i.i.d. sampled from $\mathcal{N}(0,\frac1r)$. Under this setting, \eqref{eq:proj_projback} emerges as a special case with \( c = 1 \). Decreasing \( c < 1 \) horizontally compresses \( G \), thereby coarsening the Projection Granularity, whereas increasing \( c > 1 \) does the opposite. In practice, \( c \) is chosen to be a power of two for implementation convenience, and both \( \tfrac{m}{c} \) and \( nc \) must be integers.

The overall framework is illustrated on the left of \cref{fig:VLoRP}: at each iteration, VLoRP first reshapes the original gradient matrix \( G \) (from \( \mathbb{R}^{n\times m} \) to \( \mathbb{R}^{nc \times (m/c)} \)) to adjust the Projection Granularity. The reshaped \(\tilde{G}\) is then projected into the subspace to obtain \(\tilde{G}^s\) which needs to be stored. During the parameter update, $\tilde{G}^s$ will be projected back to the original space to obtain \(\tilde{G}^o\), after which we reshape \(\tilde{G}^o\) back to update the parameter. 
Notably, in practice, the only but important difference between VLoRP and other LoRP methods is the pair of reshaping operations, underscoring its simplicity of implementation.
\subsection{Benefits of Finer Projection Granularity}
\label{sec:Projections with Finer Granularity Fairly Demonstrate Superior Efficacy}
Fundamentally, VLoRP introduces a novel degree of freedom that extends beyond the rank parameter in LoRP approaches. This raises a critical question: \textit{What level of the Projection Granularity is optimal for gradient projection?}

To address this question, it is essential to establish a fair basis for comparison first: on the one hand, for a fixed rank \( r \), employing a finer-grained projection typically improves performance but increases memory requirements; on the other hand, for a fixed granularity factor \( c \), increasing the rank similarly enhances performance while incurring additional memory costs. We thereby denote a specific pair of $(c, r)$ as a ``configuration'' of VLoRP, and define the ``Memory Budget'' \(\mathcal{M}\) as the product of \( c \) and \( r \)---for configurations sharing the same $\mathcal{M}$, the size of $\tilde{G}^s$, which is needed to be stored, is the same. Based on this, we claim that \textbf{for all configurations sharing the same memory budget, \textit{i.e.}, \( c_ir_i = c_jr_j = \mathcal{M}, \forall i, j \), opting for a finer-grained projection (larger \( c \)) is generally preferable, despite the associated reduction in rank (\( r = \frac{\mathcal{M}}{c}\))}. There are several immediate advantages: 


\textbf{Computational Efficiency}\quad 
Given the granularity factor $c$, the memory budget $\mathcal{M}$ and the reshaped gradient $\tilde{G}\in \mathbb{R}^{nc\times (m/c)}$
, the gradient projection operation $\tilde{G} \tilde{P}$ leads to a computational complexity $O\left(\frac{nm\mathcal{M}}{c}\right)$.
By choosing a finer granularity  (larger \( c \)), we can reduce the overall FLOPs involved.

\textbf{Generation Efficiency of Projection Matrix} \quad Obviously, the efficiency of this generation process depends on the dimensions of the projection matrix $\tilde{P}$ of the shape\(\tfrac{m}{c} \times \tfrac{\mathcal{M}}{c}\). As 
$c$ increases, the projection granularity becomes finer, leading to a quadratic reduction in the size of the generated matrix and hence more efficient generation.

\textbf{Numerical Error}  \quad
As $c$ increases, the numerical accuracy of the random projection tends to improve. This can be attributed to the interplay between matrix dimensions and floating-point arithmetic properties. A larger $c$ reduces the number of columns in $\tilde{G}$, thereby decreasing the number of floating-point operations per dot product in matrix multiplications. Consequently, this helps mitigate the accumulation of numerical errors, which is particularly important in low-precision training scenarios such as float16 or bfloat16, where limited mantissa precision can lead to instability. We provide an analytical experiment in \cref{appendix:numerical_error_testing} to support this observation. Furthermore, this numerical stability advantage is further amplified when using gradient accumulation or Adam-based optimizers, as these methods involve iteratively accumulating optimization states.

Besides, our experiments in \cref{sec:experiments} demonstrate that under the same memory budget 
$\mathcal{M}$, finer-grained projections can consistently outperform coarser counterpart, suggesting that $c$ play a more critical factor than rank $r$.

\subsection{Theoretical Guarantee}
\label{sec:Theoretical Guarantee of VLoRP}
Incorporating Projection Granularity into gradient estimation offers clear benefits, but it is crucial to ensure that finer-grained projections do not compromise the convergence of model training. This section provides theoretical guarantees that under a fixed memory budget \(\mathcal{M}\), varying the granularity factor \(c\) and rank \(r\) does not significantly affect the mean and variance of one-step gradient estimation or the overall convergence rate.

We consider a loss function \(\mathcal{L}: \mathbb{R}^{n \times m} \to \mathbb{R}\) defined over matrix parameters \(W \in \mathbb{R}^{n \times m}\). Fix a memory budget \(\mathcal{M}\), a granularity factor \(c\), rank \(r\), and \(G = \nabla_{W} \mathcal{L}(W)\). 
Throughout this paper, we adopt the Frobenius norm and inner product as the primary metrics for matrices and vectors.
\begin{proposition}
\label{thm:error_bound}
The gradient estimator \(G^o \in \RR^{n\times m}\) satisfies the following properties:
$$
\mathbb{E}[G^o] = G,  \quad
\mathbb{E}\|G^o - G\|^2 = \frac{m + c}{\mathcal{M}} \|G\|^2.
$$
\end{proposition}
Proposition~\ref{thm:error_bound} demonstrates that the estimator \(G^o\) is unbiased and its variance is bounded by \(O\left(\frac{m + c}{\mathcal{M}}\right)\). For LLMs, where the granularity factor \(c\) is significantly smaller than the parameter dimension \(m\), the effect of altering \(c\) on gradient approximation is minimal under a fixed memory budget \(\mathcal{M}\).
The unbiasedness and bounded variance of \(G^o\) lead to the following convergence guarantee:
\begin{theorem}\label{thm:RP_convergence}
Let \(\mathcal{L}\) be an \(L\)-smooth function with respect to the matrix-shaped parameter \(W\).
Assume the parameter updates are given by $W_{t+1} = W_t - \eta \,G_t^o$, where the step size is defined as $\eta = \tfrac{\mathcal{M}}{(m + c + \mathcal{M})L} := C$.
Then, for any \(T \geq 1\):
\[
\frac{1}{T} \sum_{t=0}^{T-1} \mathbb{E}\bigl[\|G_t\|^2\bigr] \leq \frac{2C}{T}\bigl(\mathcal{L}(W_0) - \mathcal{L}(W^*)\bigr),
\]
where \(W^*\) is a global minimizer of \(\mathcal{L}\).
\end{theorem}
Theorem~\ref{thm:RP_convergence} confirms that using random-projection-based gradient estimation with VLoRP in conjunction with SGD achieves an \(O(1/T)\) convergence rate, regardless of the granularity factor \(c\). 

\section{Adaptive Memory-Efficient Optimization}
\label{sec:4}
In this section, we first investigate potential Adam-based optimization schemes for VLoRP, then introduce a memory-efficient variant, \textbf{ProjFactor}, for the superior scheme, and finally provide theoretical convergence proof for it. We enable gradient accumulation~\cite{wang13, SmithKYL18} for the projected gradient by default, which means at each update stage, we can only access $\tilde{\bm{G}}^{s} = \sum_{i=1}^{K} \tilde{G}^s_i / K$, where \(K\) is the number of accumulation substeps.

\begin{figure*}[h]
    \centering
    \includegraphics[width=1.0\linewidth]{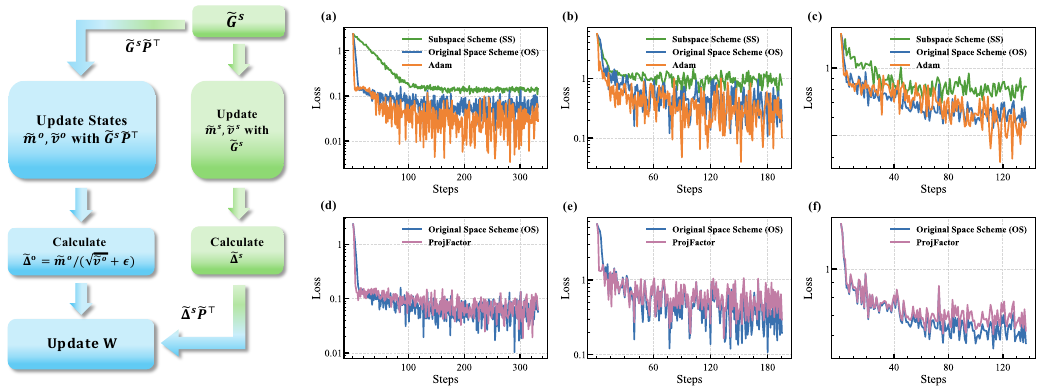} 
    \caption{\textbf{Left:} Schematic illustration of the Subspace Scheme (SS, green) operating in a learned subspace, and the Original Scheme (OS, blue) operating in the original space. 
    \textbf{Right (top row, panels a--c):} Fine-tuning loss curves of SS (green), OS (blue), and Adam (orange) on three tasks, showing that OS outperforms SS by a large margin while has a comparable performance with Adam.
    \textbf{Right (bottom row, panels d--f):} Comparison of OS (blue) and its approximate algorithm ProjFactor (purple), indicating that ProjFactor closely approximates the dynamic of OS.
    Specifically, columns (a) and (d) test on the Commonsense Reasoning task, columns (b) and (e) test on the MMLU, and columns (c) and (f) test on the GSM8K.
    }
    \label{fig:subspace_os_scheme}
\end{figure*}

\textbf{Optimization Schemes: SS vs. OS} \quad
Broadly, with access only to $\tilde{\bm{G}}^{s}$, there are two typical schemes for storing the optimization states of Adam and performing the associated updates for VLoRP, as illustrated on the left of \cref{fig:subspace_os_scheme}: (1) \textbf{Subspace Scheme (SS)} retains the optimization states $m^{s}$ and $v^{s}$ and computes the adapted gradient $\tilde{\Delta}^s_t = \tilde{m}^s_{t} / \sqrt{\tilde{v}^s_{t}}$ within the subspace, while (2) \textbf{Original Space Scheme (OS)} projects $\tilde{\bm{G}}^{s}$ back to the original space, where the optimization states $\tilde{m}^o$ and $\tilde{v}^o$ are stored and the adapted gradient is computed directly. Mathematically, the update rules for both schemes are defined as follows\footnote{\(\odot 2\) denotes element-wise squaring.}:
\begin{equation}
\begin{aligned}
\textbf{SS:} \quad
\tilde{m}^s_{t} &= \beta_1 \tilde{m}^s_{t-1} + (1 - \beta_1) \tilde{\bm{G}}^s_{t}, \\
\tilde{v}^s_{t} &= \beta_2 \tilde{v}^s_{t-1} + (1 - \beta_2) (\tilde{\bm{G}}^s_{t})^{\odot 2}, \\
\tilde{\Delta}^s_t &= \tilde{m}^s_{t} /\sqrt{\tilde{v}^s_{t}}, \\
W_{t} &= W_{t-1} - \eta \operatorname{Reshape}\left(\tilde{\Delta}^s_t\tilde{P}^{\top}, [n, m]\right); \\
\textbf{OS:} \quad 
\tilde{m}^o_{t} &= \beta_1 \tilde{m}^o_{t-1} + (1 - \beta_1) (\tilde{\bm{G}}^s_{t}\tilde{P}^{\top}), \\
\tilde{v}^o_{t} &= \beta_2 \tilde{v}^o_{t-1} + (1 - \beta_2) (\tilde{\bm{G}}^s_{t}\tilde{P}^{\top})^{\odot 2}, \\
\tilde{\Delta}^o_t &= \tilde{m}^o_{t}/\sqrt{\tilde{v}^o_{t}}, \\ 
W_{t} &= W_{t-1} - \eta \operatorname{Reshape} \left(\tilde{\Delta}^o_t, [n,m]\right).
\end{aligned}
\label{eq:ss_os_update_rules}
\end{equation}

\textbf{OS Performs Better} \quad
From \eqref{eq:ss_os_update_rules}, it can be observed that for momentum-based SGD optimizers, where the second moment $v$ is excluded, the \textbf{Subspace Scheme (SS)} and \textbf{Original Space Scheme (OS)} are equivalent, since $\tilde{m}_t^o = \tilde{m}^s_t \tilde{P}^{\top}$. However, when the nonlinear second moment is incorporated, \textbf{SS} and \textbf{OS} become different. An intuition is that the dynamics of \textbf{OS} are closer to Adam's, as (1) both \textbf{OS} and Adam adapt the gradient directly in the original space, and (2) previous works~\cite{Galore, FLoRA} have shown that the gradient of LLMs exhibits a low-rank structure, implying that $\tilde{\bm{G}} \tilde{P} \tilde{P}^\top$ is capable of preserving the primary information of $\tilde{\bm{G}}$. To substantiate it, we finetune a LLaMA2-7B on three different benchmarks and compare the loss curves of \textbf{SS}, \textbf{OS}, and the vanilla Adam. As depicted in \cref{fig:subspace_os_scheme} (right, top rows, panels a-c),  while all three schemes reduce the loss, \textbf{SS} exhibits a noticeably slower and less effective convergence compared to \textbf{OS} and Adam. In contrast, \textbf{OS} closely tracks the loss curve of Adam.


\textbf{ProjFactor: a Memory-Efficient Optimizer for VLoRP} \quad While \textbf{OS} achieves superior performance, it generally requires more memory during training compared to \textbf{SS}, as it does not reduce the memory usage of optimization states, which still occupy \(O(nm)\) space. To address this, we propose \textbf{ProjFactor} which (1) maintains $\tilde{m}^s$ instead of $\tilde{m}^o$ in the subspace and project it back to the original space when calculating $\Delta_t^o$, which is justified by the fact $\tilde{m}^o =\tilde{m}^s \tilde{P}^{\top}$; (2) follows the spirit of Adafactor~\citep{adafactor} by applying rank-$1$ decomposition on \( (\tilde{\bm{G}}^s_{t}\tilde{P}^{\top})^{\odot 2} \), which only requires the storage of two vectors, significantly reducing memory consumption while preserving effective second-moment estimates. Compared to Adafactor, which applies the rank-$1$ decomposition directly on the gradient, our method applies the rank-$r$ approximation of $\tilde{\bm{G}}$ first before the squaring and decomposition. The final algorithm is outlined in \cref{sec:alg_projfactor}. Since Adafactor can achieve performance comparable to Adam~\cite{adafactor, FLoRA}, it is reasonable to conjecture that ProjFactor can similarly match \textbf{OS}. We showcase the performance comparison between ProjFactor and \textbf{OS} in \cref{fig:subspace_os_scheme}(d-f), which aligns with our expectations.

\begin{table*}[t]
\footnotesize
\centering
\caption{Performance Comparison on Commonsense Reasoning. All models are first finetuned on the Commonsense170k~\cite{commonsense170k} dataset and then evaluated separately on 8 reasoning tasks. We set the Memory Budget $\mathcal{M} = 256$ for VLoRP, \textit{i.e.}, the product of $c$ and $r$ equal 256. For a fair comparison, we also set the rank of all other low-rank-based methods as $256$. 
}
\resizebox{1.0\textwidth}{!}{%
\begin{tabular}{l|cccccccc|c}
    \toprule
    \textbf{Methods} & \textbf{ARC\_C} & \textbf{ARC\_E} & \textbf{BoolQ} & \textbf{HellaSwag} & \textbf{OBQA} & \textbf{PIQA} & \textbf{SIQA} & \textbf{winogrande}& \textbf{Avg.} \\
    \midrule
    \textbf{Pretrain-Untuned} & 42.26 \tiny{$\pm$ 1.45} & 75.26 \tiny{$\pm$ 0.87} & 76.86 \tiny{$\pm$ 0.73} & 56.17 \tiny{$\pm$ 0.49} & 30.40 \tiny{$\pm$ 2.08} & 77.91 \tiny{$\pm$ 0.97} & 45.11 \tiny{$\pm$ 1.13} & 58.78 \tiny{$\pm$ 1.30} & 58.84 \\
    \textbf{Adam} & 47.12 \tiny{$\pm$ 1.46} & 78.55 \tiny{$\pm$ 0.83} & 82.27 \tiny{$\pm$ 0.65} & 56.08 \tiny{$\pm$ 0.49} & 33.60 \tiny{$\pm$ 2.13} & 77.45 \tiny{$\pm$ 0.96} & 52.53 \tiny{$\pm$ 1.13} & 71.69 \tiny{$\pm$ 1.25} & 62.41 \\
    \textbf{Adafactor} & 48.06 \tiny{$\pm$ 1.46} & 79.22 \tiny{$\pm$ 0.82} & 80.50 \tiny{$\pm$ 0.68} & 56.24 \tiny{$\pm$ 0.49} & 34.20 \tiny{$\pm$ 2.14} & 77.56 \tiny{$\pm$ 0.96} & 52.02 \tiny{$\pm$ 1.13} & 71.22 \tiny{$\pm$ 1.26} & 62.38 \\
    \midrule
    \textbf{LoRA$\bm{(r=256)}$} & 44.23 \tiny{$\pm$ 1.46} & 76.66 \tiny{$\pm$ 0.85} & 80.21 \tiny{$\pm$ 0.68} & 55.79 \tiny{$\pm$ 0.49} & 33.00 \tiny{$\pm$ 2.13} & 76.57 \tiny{$\pm$ 0.97} & 47.94 \tiny{$\pm$ 1.13} & 69.69 \tiny{$\pm$ 1.27} & 60.51 \\
    \textbf{Galore$\bm{(r=256)}$} & 44.13 \tiny{$\pm$ 1.45} & 76.65 \tiny{$\pm$ 0.87} & 78.23 \tiny{$\pm$ 0.72} & 57.59 \tiny{$\pm$ 0.49} & 32.00 \tiny{$\pm$ 2.09} & 77.95 \tiny{$\pm$ 0.97} & 46.01 \tiny{$\pm$ 1.13} & 69.71 \tiny{$\pm$ 1.29} & 60.28 \\
    \textbf{fira$\bm{(r=256)}$} & 44.06 \tiny{$\pm$ 1.45} & 76.63 \tiny{$\pm$ 0.87} & 78.12 \tiny{$\pm$ 0.73} & 57.69 \tiny{$\pm$ 0.49} & 32.40 \tiny{$\pm$ 2.09} & 77.78 \tiny{$\pm$ 0.97} & 46.21 \tiny{$\pm$ 1.13} & 69.89 \tiny{$\pm$ 1.29} & 60.35 \\
    \textbf{APOLLO$\bm{(r=256)}$} & 44.28 \tiny{$\pm$ 1.45} & 76.26 \tiny{$\pm$ 0.87} & 77.74 \tiny{$\pm$ 0.73} & 57.00 \tiny{$\pm$ 0.49} & 31.40 \tiny{$\pm$ 2.08} & 77.97 \tiny{$\pm$ 0.97} & 46.11 \tiny{$\pm$ 1.13} & 69.46 \tiny{$\pm$ 1.29} &  60.03 \\
    \midrule
    \textbf{VLoRP} & \\
    - \makecell[l]{\bm{$c=2^{-6},$}} \makecell[l]{\bm{$r=2^{14}$}} &  42.92 \tiny{$\pm$ 1.45} & 76.22 \tiny{$\pm$ 0.87} & 79.27 \tiny{$\pm$ 0.71} & 57.53 \tiny{$\pm$ 0.49} & 32.60 \tiny{$\pm$ 2.10} & 77.91 \tiny{$\pm$ 0.97} & 46.72 \tiny{$\pm$ 1.13} & 69.85 \tiny{$\pm$ 1.29} & 60.38 \\
    - \makecell[l]{\bm{$c=2^{-4},$}} \makecell[l]{\bm{$r=2^{12}$}}  &  43.34 \tiny{$\pm$ 1.45} & 76.26 \tiny{$\pm$ 0.87} & 79.54 \tiny{$\pm$ 0.71} & 57.58 \tiny{$\pm$ 0.49} & 32.00 \tiny{$\pm$ 2.09} & 77.64 \tiny{$\pm$ 0.97} & 46.72 \tiny{$\pm$ 1.13} & 70.01 \tiny{$\pm$ 1.29} & 60.39 \\
    - \makecell[l]{\bm{$c=2^{-2},$}} \makecell[l]{\bm{$r=2^{10}$}}  &  43.34 \tiny{$\pm$ 1.45} & 76.30 \tiny{$\pm$ 0.81} & 79.45 \tiny{$\pm$ 0.71} & 57.47 \tiny{$\pm$ 0.49} & 32.20 \tiny{$\pm$ 2.09} & 77.75 \tiny{$\pm$ 0.97} & 46.78 \tiny{$\pm$ 1.13} & 70.01 \tiny{$\pm$ 1.29} & 60.41 \\
    - \makecell[l]{\bm{$c=2^{0},$}} \makecell[l]{\bm{$r=2^{8}$}}  &  43.69 \tiny{$\pm$ 1.45} & 77.02 \tiny{$\pm$ 0.86} & 79.27 \tiny{$\pm$ 0.71} & 57.49 \tiny{$\pm$ 0.49} & 31.80 \tiny{$\pm$ 2.08} & 78.07 \tiny{$\pm$ 0.97} & 47.49 \tiny{$\pm$ 1.13} & 69.77 \tiny{$\pm$ 1.29} & 60.57 \\ 
    - \makecell[l]{\bm{$c=2^{2},$}} \makecell[l]{\bm{$r=2^{6}$}}  &  44.03 \tiny{$\pm$ 1.45} & 76.81 \tiny{$\pm$ 0.87} & 79.17 \tiny{$\pm$ 0.71} & 57.59 \tiny{$\pm$ 0.49} & 31.80 \tiny{$\pm$ 2.08} & 78.02 \tiny{$\pm$ 0.97} & 47.19 \tiny{$\pm$ 1.13} & 69.53 \tiny{$\pm$ 1.29} & 60.53 \\ 
    - \makecell[l]{\bm{$c=2^{4},$}} \makecell[l]{\bm{$r=2^{4}$}}  &  44.71 \tiny{$\pm$ 1.45} & 77.27 \tiny{$\pm$ 0.86} & 79.42 \tiny{$\pm$ 0.71} & 57.50 \tiny{$\pm$ 0.49} & 32.20 \tiny{$\pm$ 2.09} & 77.86 \tiny{$\pm$ 0.97} & 47.54 \tiny{$\pm$ 1.13} & 70.09 \tiny{$\pm$ 1.29} & 60.82 \\
    - \makecell[l]{\bm{$c=2^{6},$}} \makecell[l]{\bm{$r=2^{2}$}}  & 44.97 \tiny{$\pm$ 1.45} & 77.65 \tiny{$\pm$ 0.85} & 80.46 \tiny{$\pm$ 0.69} & 57.56 \tiny{$\pm$ 0.49} & 33.60 \tiny{$\pm$ 2.11} & 77.97 \tiny{$\pm$ 0.97} & 48.06 \tiny{$\pm$ 1.13} & 69.69 \tiny{$\pm$ 1.29} & 61.25 \\
    \rowcolor{lightblue}
    - \makecell[l]{\bm{$c=2^{8},$}} \makecell[l]{\bm{$r=2^{0}$}}  & 45.56 \tiny{$\pm$ 1.46} & 77.78 \tiny{$\pm$ 0.85} & 80.58 \tiny{$\pm$ 0.69} & 57.59 \tiny{$\pm$ 0.49} & 34.00 \tiny{$\pm$ 2.12} & 77.86 \tiny{$\pm$ 0.97} & 48.16 \tiny{$\pm$ 1.13} & 69.69 \tiny{$\pm$ 1.29} & 61.40 \\
    \bottomrule
\end{tabular}
}
\label{tab:commonsense_0shot}
\end{table*}

\textbf{Convergence Analysis} \quad
To analyze ProjFactor's convergence, we adopt the Hamiltonian descent framework, following the line of work~\cite{maddison2018hamiltonian,chen2023lion,liang2024memory,nguyen2024h}. The infinitesimal updates of Projfactor is defined as follows:
\begin{equation}\label{eq:projfactor_infinitesimal}
\begin{aligned}
&\frac{d}{dt} \tilde{m}^s_t = a(\tilde{G}^s_t-\tilde{m}^s_t); \;\hat{v}^o_t = \frac{\tilde{v}^o_{rt}\tilde{v}^o_{ct}}{\mathbf{1}_n^T \tilde{v}^o_{rt}};\\
&\frac{d}{dt}\tilde{v}^o_{rt} = b((\tilde{G}_t^o)^{\odot 2}\mathbf{1}_m-\tilde{v}^o_{rt});\\
&\frac{d}{dt}\tilde{v}^o_{ct} = b(\mathbf{1}_n^T (\tilde{G}_t^o)^{\odot 2} - \tilde{v}^o_{ct});\\
&\frac{d}{dt}W_t = \operatorname{Reshape}\left(-\tilde{m}^s_t\tilde{P}^\top  \Big /\sqrt{\hat{v}^o_t},\;[n,m]\right),
\end{aligned}
\end{equation}
where $a$, $b$ are constants.
The corresponding Lyapunov function (Hamiltonian) is defined as 
\begin{equation*}
\cH(W,\tilde{m}^s,\tilde{v}^o_{r},\tilde{v}^o_c) = \cL(W)+ \frac{1}{2a}\left\langle \tilde{m}^s, \frac{\tilde{m}^s}{\sqrt{\hat{v}^o}}\right\rangle.
\end{equation*}

Then we have the following convergence guarantee:
\begin{theorem}\label{thm:convergence}
Suppose the functions in system \eqref{eq:projfactor_infinitesimal} are continuously differentiable. Under mild assumptions, we have
\begin{enumerate}[(1)]
    \item For $(W_t, \tilde{m}^s_t, \tilde{v}_{rt}^o, \tilde{v}^o_{ct})$ satisfying \eqref{eq:projfactor_infinitesimal}, $$\frac{d}{dt}\cH(W_t, \tilde{m}^s_t, \tilde{v}_{rt}^o, \tilde{v}^o_{ct})\leq 0.$$
    \item Any bounded solution \((W_t, \tilde{m}^s_t, \tilde{v}_{rt}^o, \tilde{v}^o_{ct})_t\) of \eqref{eq:projfactor_infinitesimal} converges to a stationary point of \(\mathcal{L}(W)\) as \(t \to \infty\).
\end{enumerate}
\end{theorem}
Theorem~\ref{thm:convergence} presents a Hamiltonian interpretation of Projfactor, indicating that the Lyapunov function, namely, the ``energy'' of the system, decreases monotonically over time. In addition, it ensures that ProjFactor stabilizes at a local optimum, provided the step sizes are sufficiently small. We defer the detailed assumption and proof to Appendix~\ref{appendix:projfactor-convergence}.

\section{Experiments}
\label{sec:experiments}
In this section, we evaluate the effectiveness of VLoRP across multiple finetuning tasks, demonstrating its competitiveness with state-of-the-art baselines. More experimental studies and implementation details can be found in \cref{appendix:empirical_analysis}.

\textbf{Datasets} \quad We present a comprehensive evaluation of our proposed approaches using three benchmarks: (1) \textit{Commonsense Reasoning}, which covers 8 reasoning tasks including BoolQ~\cite{BoolQ}, PIQA~\cite{PIQA}, SIQA~\cite{SIQA}, Hellaswag~\cite{hellaswag}, WinoGrande~\cite{winogrande}, ARC-e~\cite{arc}, ARC-c~\cite{arc}, and OBQA~\cite{OBQA}; (2) The \textit{MMLU} benchmark~\citep{mmlu}, which encompasses a wide range of subjects including Humanities, STEM, Social Sciences, and Other fields; (3) Finally, the \textit{GSM8K} dataset~\cite{gsm8k}, which is a dataset of 8.5K high-quality problems of mathematics.

\textbf{Baselines} \quad We compare VLoRP with 7 state-of-the-art baselines, covering full-parameter finetuning, LoRA, and LoRP methods. Specifically, we compare with \textbf{Pretrain-Untuned}, which represents the basic performance of the pre-trained model without finetuning, \textbf{Adam}~\citep{kingma2014adam}, \textbf{Adafactor}~\citep{adafactor}, \textbf{LoRA}~\citep{lora}, \textbf{Galore}~\citep{Galore}, \textbf{fira}~\cite{fira}, and \textbf{APOLLO}~\cite{appollo}. 

\textbf{Experimental Settings} \quad We use the LLaMA2-7B model with the bfloat16 data type as the primary testbed for all methods. Each method is guaranteed that every token in the training set is encountered at least once. Gradient accumulation and activation checkpointing~\cite{Chen16} are enabled. VLoRP is optimized with ProjFactor by default.

\textbf{Results and Analysis} \quad We report the performance of all methods on: (1) commonsense reasoning tasks in \cref{tab:commonsense_0shot}, (2) the MMLU benchmark in \cref{tab:mmlu_results}, and (3) the GSM8k task in \cref{fig:GSM8k_exp}. In general, all finetuning methods yield significant performance improvements compared to the untuned model. Among the finetuning baselines, Adam and Adafactor generally achieve the highest performance across tasks. Notably, Adafactor delivers comparable even superior results to Adam, consistent with previous findings~\cite{adafactor, FLoRA}. Furthermore, the proposed VLoRP, optimized with ProjFactor, not only rivals the performance of the strongest baselines of efficient finetuning but in many cases surpasses them. On top of it, under a fixed memory budget, configurations of VLoRP with finer Projection Granularity (\textit{i.e.}, larger factor \(c\) albeit smaller $r$) tend to achieve higher average accuracy. For instance, the finest-grained VLoRP configuration \(\bigl(c=2^8, r=2^0\bigr)\) achieves the highest scores of 61.40, 55.83, and 29.42 on Commonsense Reasoning, MMLU, and GSM8k, respectively, among all tested configurations of \(c\) and \(r\) with the same memory budget. This suggests that Projection Granularity may play a more critical role than rank in balancing memory efficiency and performance.

\begin{table}[h!]
\caption{Performance Comparison on the MMLU benchmark. We also set the Memory Budget $\mathcal{M} = 256$ for VLoRP. The best-performing configurations are highlighted.}
\resizebox{1.0\linewidth}{!}{%
\begin{tabular}{l|cccc|c}
    \toprule
    \textbf{Methods} & \textbf{Hum.} & \textbf{STEM} & \textbf{S. Sci.} & \textbf{Other} & \textbf{ALL} \\
    \midrule
    \textbf{Pretrain-Untuned} & 39.54 & 34.85 & 49.14 & 47.73 & 42.53 \\
    \textbf{Adam} & 52.63 & 44.92 & 65.31 & 62.45 & 56.01 \\
    \textbf{Adafactor} & 53.01 & 46.55 & 66.13 & 56.87 & 56.80 \\
    \midrule
    \textbf{LoRA $\bm{(r=256)}$} &  50.74 & 43.61 & 60.93 & 59.73 & 53.55 \\
    \textbf{Galore $\bm{(r=256)}$} & 50.22 & 42.61 & 60.25 & 59.52 & 52.93 \\
    \textbf{fira $\bm{(r=256)}$} & 49.85 & 42.49 & 60.31 & 59.55 & 52.83 \\
    \textbf{APOLLO $\bm{(r=256)}$} & 49.12 & 41.42 & 57.71 & 56.91 & 51.13 \\
    \midrule
    \textbf{VLoRP} & \\
    - \makecell[l]{\bm{$c=2^{-6},$}} \makecell[l]{\bm{$r=2^{14}$}} & 47.06 & 39.54 & 56.33 & 56.71 & 49.65 \\
    - \makecell[l]{\bm{$c=2^{-4},$}} \makecell[l]{\bm{$r=2^{12}$}} & 49.88 & 42.50 & 58.44 & 58.69 & 52.22 \\
    - \makecell[l]{\bm{$c=2^{-2},$}} \makecell[l]{\bm{$r=2^{10}$}} & 50.04 & 41.21 & 58.74 & 58.63 & 52.01 \\
    - \makecell[l]{\bm{$c=2^{0},$}} \makecell[l]{\bm{$r=2^{8}$}} & 50.62 & 43.23 & 61.08 & 60.83 & 53.64 \\
    - \makecell[l]{\bm{$c=2^{2},$}} \makecell[l]{\bm{$r=2^{6}$}} & 50.93 & 44.31 & 61.04 & 60.51 & 53.91 \\
    - \makecell[l]{\bm{$c=2^{4},$}} \makecell[l]{\bm{$r=2^{4}$}} & 50.85 & 44.22 & 61.08 & 60.91 & 54.05 \\
    - \makecell[l]{\bm{$c=2^{6},$}} \makecell[l]{\bm{$r=2^{2}$}} & 51.94 & 44.53 & 63.53 & 62.85 & 55.42 \\
    \rowcolor{lightblue}
    - \makecell[l]{\bm{$c=2^{8},$}} \makecell[l]{\bm{$r=2^{0}$}} & 52.29 & 46.18 & 64.05 & 62.24 & 55.83 \\
    \bottomrule
\end{tabular}
}
\label{tab:mmlu_results}
\end{table}

\begin{figure}[h]
    \centering
    \includegraphics[width=1.0\linewidth]{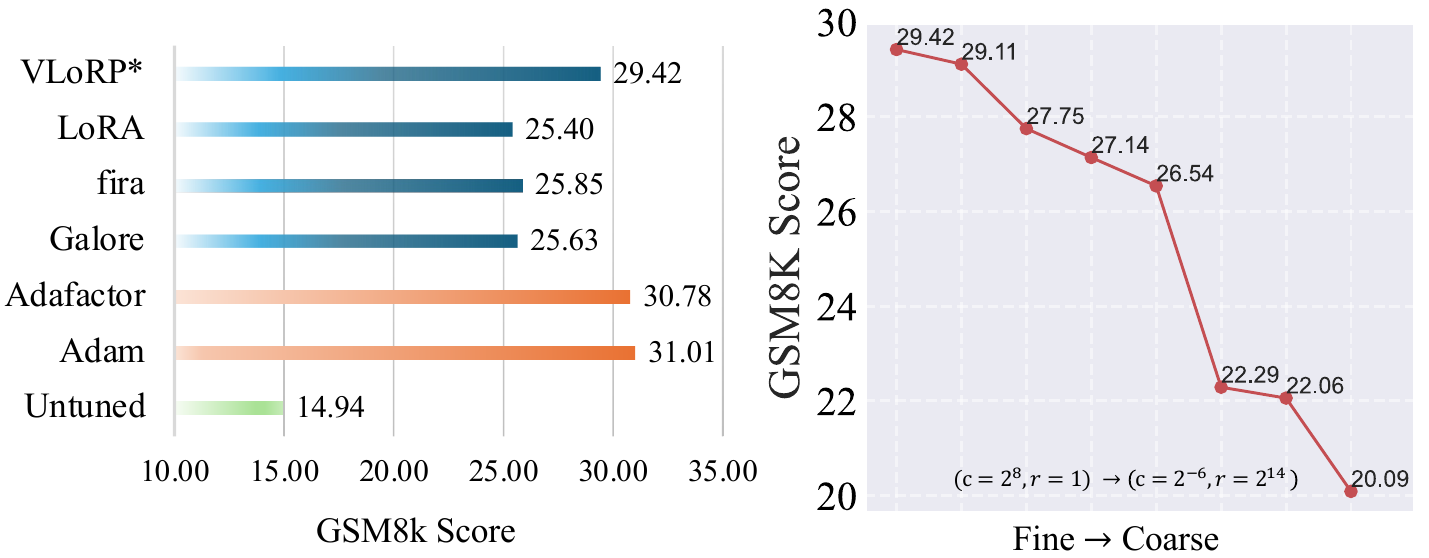} 
    \caption{\textbf{Left:} Performance comparison of different methods on GSM8K. \textbf{Right:} Performance comparison among the configurations of VLoRP with $\mathcal{M}=256$. The x-axis indicates configurations from fine to coarse (left to right).}
    \label{fig:GSM8k_exp}
\end{figure}

\begin{figure}[h]
    \centering
    \includegraphics[width=1\linewidth]{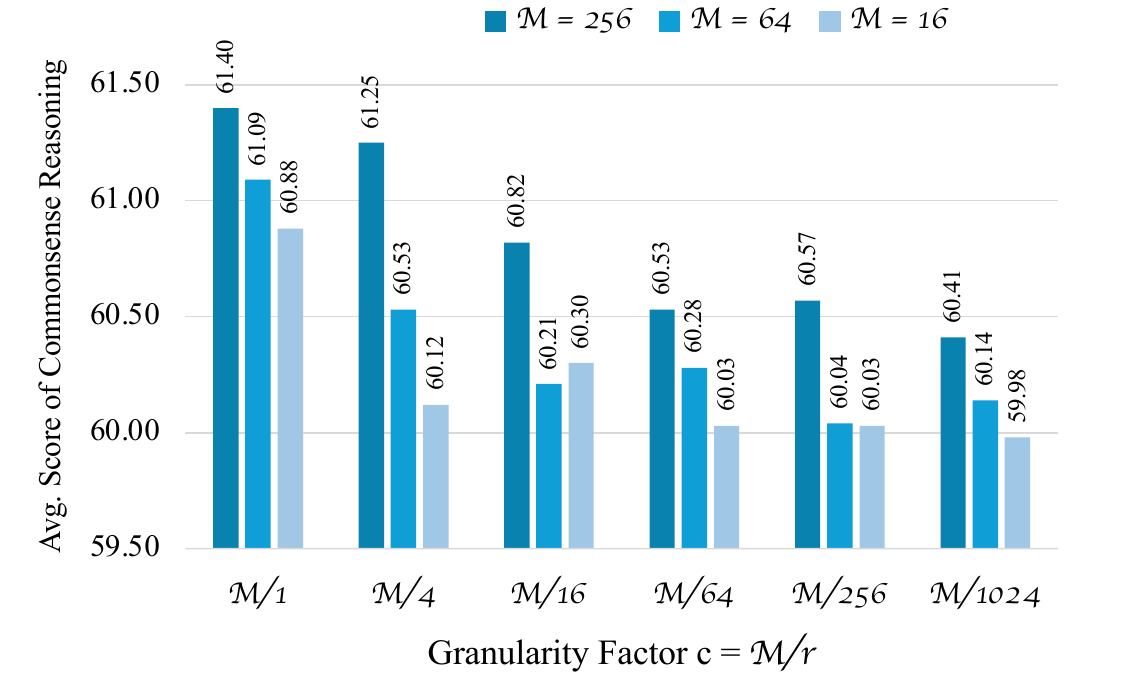} 
    \caption{Performance Evaluation for Different Projection Granularities under Varying Memory Budgets on Commonsense. $\mathcal{M}$ denotes the memory budget. All subcolumns inside the same category $\mathcal{M}/r$ share the same rank $r$.}
    \label{fig:ablation_of_memory_budget}
\end{figure}
\textbf{Different Memory Budgets}\quad Besides, in \cref{fig:ablation_of_memory_budget}, we present the performance of VLoRP under different projection granularities across varying memory budgets. 
Overall, the results indicate that configurations with finer projection granularities consistently outperform other coarser configurations, regardless of the memory budget level. Furthermore, when the rank is fixed (the three subcolumns within each column), the performance improves with finer granularities. 
\begin{figure}[h]
    \centering
    \includegraphics[width=0.62\linewidth]{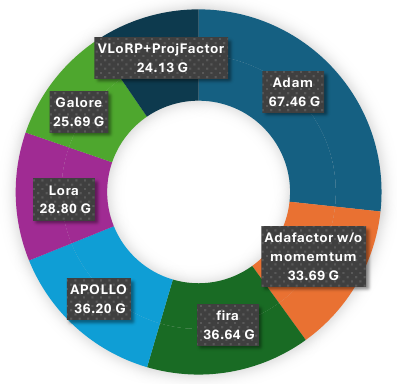} 
    \caption{GPU Memory Usage Comparison on LLaMA2-7B Model with Batch Size 16 and Max Length 1024.}
    \label{fig:memory_ring}
\end{figure}

\textbf{Memory Analysis}  \quad
We illustrate the memory usage of various methods in \cref{fig:memory_ring}, emphasizing that VLoRP achieves the most significant memory reduction compared to baseline approaches. When incorporating gradient accumulation, certain memory-efficient methods, such as fira and APOLLO, fail to achieve memory usage below Adafactor. This limitation stems from their reliance on the full-rank gradient for update computations, which primarily reduces the memory associated with optimization states but not the gradient itself.
In contrast, optimizing VLoRP with ProjFactor fundamentally reduces memory consumption by solely operating on the projected gradients. For a parameter matrix \( W \in \mathbb{R}^{n \times m} \), the total memory required for VLoRP is \( O(mn + 2n\mathcal{M} + n + m) \), where \( \mathcal{M}=c\cdot r \) is the memory budget. On the other hand, LoRA requires significantly more storage, \( O(mn + 4m\mathcal{M} + 4n\mathcal{M}) \), due to its need for additional low-rank matrices and optimization states.

\section{Conclusion}
In this work, we introduce VLoRP, a memory-efficient finetuning method for LLMs, which implements low-rank gradient projections with varying granularities. By adjusting the projection granularity alongside rank, VLoRP offers a more nuanced control over the trade-off between memory efficiency and performance. Furthermore, we present ProjFactor, an adaptive optimizer that reduces memory consumption while maintaining competitive performance. Theoretical analysis and empirical results confirm the effectiveness of our approach, making it a promising solution for practical deployment in memory-constrained settings.

\newpage
\section*{Impact Statement}
This paper presents work whose goal is to advance the field of Machine Learning by improving the memory efficiency and performance of LLM training. Specifically, our proposed method facilitates more efficient utilization of memory, which is particularly critical in the context of the current scarcity of computational resources. Furthermore, considering the significant energy consumption associated with training large models, particularly in terms of electricity and its environmental impact, a more efficient algorithm like ours represents a substantial step toward alleviating the energy costs of model training. 




\bibliography{reference}
\bibliographystyle{icml2025}

\newpage
\appendix
\onecolumn
\begin{table}[ht]
\renewcommand{\arraystretch}{1.5}
\centering
\footnotesize
\caption{A detailed table for notations used in this paper.}
\resizebox{1.0\linewidth}{!}{%
\begin{tabular}{c|l|p{9cm}}
\toprule
\textbf{Symbol} & \textbf{Definition} & \textbf{Description} \\
\midrule
\((\cdot)^s\) & Subspace tag 
                           & Distinguish variables in the projected low-dimensional space from those in the original space. \\
\((\cdot)^o\) & Original Space tag & Distinguish variables in the original space from those in the projected low-dimensional space; \\
$(\cdot)_{t}$ & Update step tag & Denotes the specific step of current variables, for example, the gradient $G$ at the $t$-th update step can be denoted as $G_{t}$; \\ 
$(\cdot)^{\odot 2}$ & Element-wise Squaring & The element-wise squaring of a matrix; \\
$\|\cdot\|$ & Frobenius norm & Taking the square root of the summation of each squared element;\\
$\langle \cdot,\cdot\rangle$ & Frobenius inner product & Inner product induced by Frobenius norm;\\
\(W\) & The Parameters Matrix 
      & The shape is assumed as \(n \times m\) with \(n \geq m\); \\
$\mathcal{L}$ & The loss function & The loss function of the training procedure; \\
\(P, \tilde{P}\) & Projection Matrix 
      & Randomly sampled from a normal distribution $\mathcal{N}(0, \frac{1}{r})$ unless otherwise stated. The shape of $P$ is $m\times r$ with $r\ll \min\{m,n\}$, while the shape of $\tilde{P}$ is $(m/c)\times r$ or $(m/c)\times (\mathcal{M}/c)$ ;\\
\(c\) & granularity factor 
      & The parameter $c$ is a hyperparameter of VLoRP that controls the granularity of projections. For instance, setting $c=2$ reduces the granularity of projections by half for the entire model. For vanilla low-rank-based methods like LoRA or Galore, their $c$ is equal to 1;\\
\(r\) & rank
      &  The parameter $r$ is a hyperparameter of VLoRP and other low-rank based memory-efficient methods, such as LoRA and Galore;\\
\(\mathcal{M}\) & Memory Budget
      & We introduce the memory budget, denoted as \( \mathcal{M} \), for VLoRP to facilitate the comparison between different configurations of \( (c, r) \). The memory budget \( \mathcal{M} \) is defined as the product of \( c \) and \( r \), as both parameters jointly influence the memory requirements during LLM training. For other low-rank-based methods, where \( c = 1 \), the rank is set to \( r = \mathcal{M} \);\\
\(G, G^{s}, G^{o}\) & Gradient 
      & Gradient computed for a single forward-backward step. The shape is equal to \(W\); $G^{s}$ represents the projected gradient, \textit{i.e.} $G^{s} = GP$, and the shape of $G^{s}$ is $n \times r$; $G^{o}$ represents the projected-back gradient, \textit{i.e.} $G^{o} = GPP^{\top}$, where $r$ is the rank, and the shape $G^{o}$ is $n \times m$; \\
\(\tilde{G}, \tilde{G}^{s}, \tilde{G}^{o}\) & Reshaped Gradient 
      & The reshaped version of gradients. The shape of $\tilde{G}$ is equal to \(nc\times (m/c)\); The shape of $G^{s}$ is $nc \times r$; The shape $\tilde{G}^{o}$ is $nc \times (m/c)$; \\
\(\bm{G},\bm{G}^s,\bm{G}^o\) & Accumulated Gradient 
               & Accumulation of gradients over multiple forward-backward steps, that is $\bm{G}_ = \sum_{i=1}^{k}G_i$ where $k$ is number of accumulation steps; $\bm{G}^{s}$/ $\bm{G}^{o}$ represents the projected/projected-back gradient, \textit{i.e.} $\bm{G}^{s} = \bm{G}P$; $\bm{G}^{o} = \bm{G}PP^{\top}$; 
               \\
$m, m^s, m^o$ & First moment of Adam & $m_{t} = \beta_1 m_{t-1} + (1-\beta_1)\bm{G}_{t}$, where $\beta_1$ represents the coefficient. $m^s$ represents the states stored in the subspace, \textit{i.e.} $m^s_{t} =\beta_1 m^s_{t-1} + (1-\beta_1)\bm{G}^s_{t} $; $m^o$ represents the states stored in the original space, \textit{i.e.} $m^o_{t} =\beta_1 m^o_{t-1} + (1-\beta_1)\bm{G}^o_{t} $; \\
$v, v^s, v^o$ & Second moment of Adam & $v_{t} = \beta_2 v_{t-1} + (1-\beta_2)(\bm{G}_{t})^{\odot 2}$; $v^s_{t} =\beta_1 v^s_{t-1} + (1-\beta_1)(\bm{G}^s_{t})^{\odot 2} $; $v^o_{t} =\beta_1 v^o_{t-1} + (1-\beta_1)(\bm{G}^o_{t})^{\odot 2}$. \\
\bottomrule
\end{tabular}
}
\label{tab:notation}
\end{table}

\clearpage
\section{Notations}\label{notations}
In \cref{tab:notation}, we provide the notations used in the main body and appendix of this paper. In case of any discrepancies between the definitions of the symbols in the table and those in the text, the definitions in the text should be followed.

Next, in \cref{sec:alg_projfactor}, we discuss our optimization scheme, specifically the algorithmic details of VLoRP with ProjFactor. \cref{appendix: theorem_proofs} presents the proofs for all theoretical results and propositions introduced in the main text. \cref{appendix:empirical_analysis} further provides several analytical experiments and ablation studies with specific implementation details. Finally, in \cref{appendix:related_works}, we conduct an in-depth discussion of related works.

\section{Algorithm of ProjFactor for VLoRP}
\label{sec:alg_projfactor}
In \cref{ag:projfactor}, we present the final algorithm employed in our study---optimizing VLoRP with ProjFactor. Formally, given a learning rate $\eta$, a parameter matrix $W$, rank $r$, granularity factor $c$, and resampling gap $\tau$, we first initialize $\tilde{m}^s$, $\tilde{v}^o_{r}$, and $\tilde{v}^o_{c}$, which serve as optimization states and need to be stored throughout training. Next, at the beginning of each update iteration, a zero matrix $\tilde{\bm{G}}^s \in \bm{0}^{nc\times r}$ is created to store the projected accumulated gradient. Subsequently, $K$ substeps of forward-backward propagation are performed with each gradient $\nabla_W \mathcal{L}(\mathcal{B}_i)$ ($\mathcal{B}_i$ denotes the mini-batch data of the accumulation step $i$) projected, reshaped, and accumulated in $\tilde{\bm{G}}^s$. After the gradient projection and accumulation, in line 13, we update the state $\tilde{m}^s$ of the first moment, while in lines 14–16, we first project $\tilde{\bm{G}}^s_t$ back to the original space via $\tilde{\bm{G}}^s_{t}\tilde{P}^{\top}$, and then perform the second moment update through factorization~\cite{adafactor}. It is important to note that $\tilde{m}^s$ is first projected back to the original space using $\tilde{m}^s_{t}\tilde{P}^{\top}$ prior to calculating $\Delta^{o}_{t}$, in alignment with the \textbf{Original Space Scheme}. The following relation justifies this:
\[
\tilde{m}^{o}_{t} = \beta_1 \ \tilde{m}^o_{t-1} + (1 - \beta_1) \tilde{\bm{G}}^{o}_{t} = \sum_{\tau=1}^{t} \beta_1^{t - \tau} (1 - \beta_1) \tilde{\bm{G}}^{o}_{\tau} = \sum_{\tau=1}^{t} \beta_1^{t - \tau} (1 - \beta_1) \left(\tilde{\bm{G}}^{s}_{\tau} \tilde{P}^{\top} \right) = \tilde{m}_t^{s}\tilde{P}^{\top}.
\]
Additionally, before updating $W$ in line 17, we multiply a bias correction term $\frac{1 - \beta_2^t}{1 - \beta_1^t}$, as in Adam~\cite{kingma2014adam} and Adafactor~\cite{adafactor}. Besides, with the same $\zeta$, the result of generation $\tilde{P}$ is equal.
\begin{algorithm}
\caption{\textbf{ProjFactor for VLoRP}}
\begin{algorithmic}[1]
\label{ag:projfactor}
\STATE \textbf{Input:} learning rate $\eta$, parameter $W\in\mathbb{R}^{n\times m}$, rank $r$, granularity factor $c$, resampling gap $\tau$;
\STATE \textbf{Initialize:} $\tilde{m}^s \gets \bm{0}^{nc\times r}$, $\tilde{v}^{o}_{r}\gets \bm{0}^{nc\times 1}, \tilde{v}^{o}_{c}\gets \bm{0}^{1\times \frac{m}{c}}$;
\WHILE{not converged}
    \IF{$t\ \operatorname{mod} \tau  == 0$}
    \STATE Resampling random seed $\zeta$;
    \ENDIF
    \STATE $\tilde{\bm{G}}^{s}_{t} \gets \bm{0}^{nc\times r}$;
    \FOR{$i =1,2,\dots,K$}
        \STATE Sample a mini-batch $\mathcal{B}_i$, calculate $\mathcal{L}(B_i)$ and then generate $\tilde{P}\in \mathbb{R}^{\frac{m}{c}\times r}$ with $\tilde{p}_{ij} \sim \mathcal{N}_{\zeta}(0, 1/r)$
        \STATE $\tilde{\bm{G}}^{s}_{t} \gets \tilde{\bm{G}}^{s}_{t} +  \operatorname{Reshape(}\nabla_W \mathcal{L}(\mathcal{B}_i)/K, \left[nc, \frac{m}{c}\right])\tilde{P}$;
    \ENDFOR
    \STATE Generate $\tilde{P}\in \mathbb{R}^{\frac{m}{c}\times r}$ with $\tilde{p}_{ij} \sim \mathcal{N}_{\zeta}(0, 1/r)$;
    \STATE $\tilde{m}^{s} \gets \beta_1 \tilde{m}^{s} + (1 - \beta_1) \tilde{\bm{G}}^{s}_{t}$;
    \STATE $\tilde{v}^{o}_{r} \gets \beta_2 \tilde{v}^{o}_{r} + (1 - \beta_2)\left(\tilde{\bm{G}}^s_{t}\tilde{P}^{\top}\right)^{\odot 2} \bm{1}_{m}$;
    \STATE $\tilde{v}^{o}_{c} \gets \beta_2 \tilde{v}^{o}_{c} + (1 - \beta_2)\bm{1}_{n}^{\top}\left(\tilde{\bm{G}}^s_{t}\tilde{P}^{\top}\right)^{\odot 2}$;
    \STATE $\Delta^{o}_{t} = \operatorname{Reshape}\left(\tilde{m}^s_{t}\tilde{P}^{\top} / \left(\sqrt{\frac{\tilde{v}^{o}_{r}\tilde{v}^{o}_{c}}{\bm{1}_{n}^{\top}\tilde{v}^{o}_{r}}}+ \epsilon\right), [n, m]\right)$;
    \STATE $W \gets W -  \eta \ \frac{1 - \beta_2^t}{1 - \beta_1^t} \Delta^{o}_{t}$ ;
    \STATE $t \gets t + 1$;
\ENDWHILE
\STATE \textbf{Output:} Optimized $W$.
\end{algorithmic}
\end{algorithm}

\section{Proof of Theorems}
\label{appendix: theorem_proofs}
In this section, we provide the proof for Proposition~\ref{propn:random_projection_mv_appendix}, Theorem~\ref{thm:RP_convergence}, and Theorem~\ref{thm:convergence}. Throughout the proofs, we adopt the Frobenius norm and inner product as the primary metrics for matrices and vectors.
\subsection{Proof of Proposition~\ref{thm:error_bound}}
We first concretely compute the variance of the forward gradient estimator with Gaussian samples for vector-input functions. A similar case is studied in \citet{FG2U}, where the samples are i.i.d. Rademacher distribution.
\begin{lemma}\label{lemma:error_bound_fg_appendix}
Let $h:\mathbb{R}^N \to \mathbb{R}$ be a differentiable function, and fix any point $\phi\in\mathbb{R}^N$.  Let
$g := \nabla_\phi h(\phi)\in\mathbb{R}^{1\times N}$
be the gradient viewed as a row vector.  Suppose we draw $b$ i.i.d.\ samples $v_1,\dots,v_b \sim \mathcal{N}(0,I_N)$, with each $v_i$ being an $N\times 1$ column vector.  Define the forward gradient estimator of size $b$ by
\[\hat{g}=\frac{1}{b}\sum_{i=1}^b gv_iv_i^\top,\]
then its mean squared error is
\[\mathbb{E}\Bigl[\|\hat{g} - g\|^2\Bigr]=\frac{N+1}{b}\|g\|^2.\]
\end{lemma}

\begin{proof}
\textbf{Unbiasedness:}
Consider a single random sample $v \sim \mathcal{N}(0,I_N)$. Since $\mathbb{E}[vv^\top] = I_N$, we have
\[\mathbb{E}\bigl[gvv^\top\bigr]=g \mathbb{E}[vv^\top]=g.\]
By linearity of expectation, averaging $b$ such i.i.d.\ samples preserves unbiasedness:
\[\mathbb{E}\biggl[\frac{1}{b}\sum_{i=1}^b \bigl(gv_iv_i^\top\bigr)\biggr]=g.\]

\textbf{Variance:}
First, consider one-sample estimator $gvv^\top$. 
Let $\alpha := gv \in \mathbb{R}$. Then
$g(vv^\top) = (gv)v^\top = \alpha v^\top$.
By moment identities, we have 
\[\mathbb{E}\bigl[\|gvv^\top - g\|^2\bigr]=\mathbb{E}\bigl[\|\alpha v^\top - g\|^2\bigr]= \mathbb{E}\bigl[\|\alpha v^\top\|^2-2\langle\alpha v^\top, g\rangle+\|g\|^2\bigr]=(N+1)\|g\|^2.\]
With $b$ i.i.d. samples, 
\[\mathbb{E}\bigl[\|\hat{g} - g\|^2\bigr]=
\frac{1}{b^2}\sum_{i=1}^b\mathbb{E}\bigl[\|gv_iv_i^\top - g\|^2\bigr]=
\frac{1}{b^2} \bigl(b\cdot (N+1)\,\|g\|^2\bigr)=
\frac{N+1}{b}\|g\|^2.\]
\end{proof}

We consider a loss function \(\mathcal{L}: \mathbb{R}^{n \times m} \to \mathbb{R}\) defined over matrix parameters \(W \in \mathbb{R}^{n \times m}\). Fix a memory budget \(\mathcal{M}\), a granularity factor \(c\) and rank \(r\) with $\cM=cr$.  Let \(G = \nabla_{W} \mathcal{L}(W)\), and let \(\tilde{G}\), \(\tilde{G}^s\), \(\tilde{G}^o\), and \(G^o\) follow the definitions in \eqref{eq:RDP}. 

\begin{proposition}
\label{propn:random_projection_mv_appendix}
The gradient estimator \(G^o\) satisfies the following properties:
\begin{gather}
\mathbb{E}[G^o] = G, \\
\mathbb{E}\|G^o - G\|^2 = \frac{m + c}{\mathcal{M}} \|G\|^2.
\end{gather}
\end{proposition}
\begin{proof}
Recall that $G$ is reshaped into 
$\tilde{G} \in \mathbb{R}^{(n c)\times (m/c)}$, and then randomly approximated by
  $\tilde{G}^o=\tilde{G}\tilde{P}\tilde{P}^\top$, where  
$\tilde{P} \in \mathbb{R}^{\frac{m}{c}\times r}$ have i.i.d.\ Gaussian columns
$v_i \in \mathbb{R}^{m/c}, i=1,\dots,r$ with $r = \cM/c$, such that \[\tilde{P}\tilde{P}^\top=\sum_{i=1}^{r}v_iv_i^\top.\] And finally,
 $\tilde{G}^o$ is reshaped back to size $n\times m$ to obtain $G^o$.

\textbf{Unbiasedness:}
Row-by-row application of Lemma~\ref{lemma:error_bound_fg_appendix} shows each row of $\tilde{G}^o$ is an unbiased
estimator of the corresponding row of $\tilde{G}$. Hence
$\mathbb{E}[\tilde{G}^o] =\tilde{G}$. Consequently, 
$ \mathbb{E}[G^o] =G$, for reshaping does not affect the bias.

\textbf{Variance:}
We first write 
\begin{equation*}\label{eq:gradient_expansion_appendix}
\begin{aligned}
\tilde{G}^o=\frac{1}{r} \sum_{i=1}^{r} \tilde{G}  v_i v_i^\top 
=\left(
  \begin{array}{c}
    \frac{1}{r}\sum\limits_{i=1}^{r} (\tilde{G}_{1,:}v_i ) v_i^\top \\
    \vdots \\
    \frac{1}{r}\sum\limits_{i=1}^{r} (\tilde{G}_{nc,:}v_i) v_i^\top \\
  \end{array}
\right).
\end{aligned}
\end{equation*}
By Lemma~\ref{lemma:error_bound_fg_appendix}, each row (dimension $d=\frac{m}{c}$) with $b = \frac{r}{c}$ samples
has variance 
$$\mathbb{E}\left[\left\|\frac{1}{r}\sum\limits_{i=1}^{r} (v_i^\top \tilde{G}^\top_{i,:} ) v_i^\top\right\|^2\right]=\frac{\frac{m}{c}+1}{r}\|\tilde{G}_{i,:}\|^2=\frac{m+c}{\cM}\|\tilde{G}_{i,:}\|^2,$$
where the last equality uses $\cM=cr$. Subsequently, summing over all rows in $\tilde{G}$ yields
\[
  \mathbb{E}\bigl[\|\tilde{G}^o - \tilde{G}\|^2\bigr]=
  \frac{m + c}{\mathcal{M}}\|\tilde{G}\|^2.
\]
Because reshaping does not change the Frobenius norm,
\[\mathbb{E}\bigl[\|G^o - G\|^2\bigr]=
  \frac{m + c}{\,\mathcal{M}\,}\,\|G\|^2.\]
Thus $G^o$ is unbiased with the stated mean-squared error.
\end{proof}

\subsection{Proof of Theorem~\ref{thm:RP_convergence}}
\begin{theorem}
\label{thm:RP_convergence_apx}
Let \(\mathcal{L}\) be an \(L\)-smooth function with respect to the matrix-shaped parameter \(W\), \textit{i.e.},
\[
\mathcal{L}(W') \leq \mathcal{L}(W) + \bigl\langle G, W' - W \bigr\rangle + \frac{L}{2} \|W' - W\|^2
\]
for any \(W, W' \in \mathbb{R}^{n \times m}\). Assume the parameter updates are given by:
\[
W_{t+1} = W_t - \eta G^o_t,
\]
where the step size is defined as $\eta = \frac{\mathcal{M}}{(m + c + \mathcal{M})L} \triangleq C$.
Then, for any \(T \geq 1\):
\[
\frac{1}{T} \sum_{t=0}^{T-1} \mathbb{E}\bigl[\|G_t\|^2\bigr] \leq \frac{2C}{T}\bigl(\mathcal{L}(W_0) - \mathcal{L}(W^*)\bigr),
\]
where \(W^*\) is a global minimizer of \(\mathcal{L}\).
\end{theorem}

\begin{proof}
By $L$-smoothness and the update rule $W_{t+1} = W_t - \eta G^o_t$, we have
\begin{equation*}
\begin{aligned}
     \cL(W_{t+1})
     &\leq   \cL(W_t)
  +
    \Bigl\langle
      G_t,
      W_{t+1}-W_t
    \Bigr\rangle
  +
  \frac{L}{2}\bigl\|W_{t+1}-W_t\bigr\|^2\\
  &=
  \cL(W_t)
  -\eta
    \Bigl\langle
      G_t,
      G^o_t
    \Bigr\rangle
  +
  \frac{L\eta^2}{2}\bigl\|G^o_t\bigr\|^2\\
  & =
  \cL(W_t)
  -\eta\Bigl\langle
      G_t,
      G^o_t
    \Bigr\rangle
  +\frac{L\eta^2}{2}\bigl\|G^o_t- G_t\bigr\|^2+\frac{L\eta^2}{2} \bigl\|G_t\bigr\|^2+L\eta^2\Bigl\langle G^o_t-G_t,G_t\Bigr\rangle.
\end{aligned}
\end{equation*}

Taking the expectation over the randomness in $G^o_t$ conditioning on $W_t$, and then using the unbiasedness of $G^o_t$, we get
\begin{equation*}
\begin{aligned}
\mathbb{E}\bigl[\cL(W_{t+1})\vert W_t\bigr]
  &\le
\mathbb{E}\bigl[\cL(W_t)\vert W_t\bigr]
  -\left(\eta-\frac{L\eta^2}{2}\right)    \mathbb{E}\bigl[\|G_t\|^2 \vert W_t\bigr]
 +\frac{L\eta^2}{2}    \mathbb{E}\bigl[\|G^o_t-G_t\|^2 \vert W_t\bigr]\\
 &= \cL(W_t)
  -\left(\eta-\frac{L\eta^2}{2}\right)    \|G_t\|^2
 +\frac{L\eta^2}{2}    \mathbb{E}\bigl[\|G^o_t-G_t\|^2 \vert W_t\bigr].
\end{aligned}
\end{equation*}

By Proposition~\ref{propn:random_projection_mv_appendix},
$\mathbb{E}\bigl[\|G^o_t-G_t \|^2\vert W_t\bigr]
=\frac{m+c}{\cM}\|G_t\|^2.$
Hence
\begin{equation*}
\mathbb{E}\bigl[\cL(W_{t+1})\bigr|W_t]
  \le
\cL(W_t)
-\left(\eta-\frac{(m+c+\cM)L\eta^2}{2\cM}\right)
\|G_t\|^2.
\end{equation*}
Taking expectation over $W_t$, we further write
\begin{equation}
\label{eq:descent_ineq}
\mathbb{E}\bigl[\cL(W_{t+1})\bigr]
  \le
\mathbb{E}\bigl[\cL(W_t)\bigr]
-\left(\eta-\frac{(m+c+\cM)L\eta^2}{2\cM}\right)
\mathbb{E}\bigl[\|G_t\|^2 \bigr].
\end{equation}

Summing \eqref{eq:descent_ineq} from $t = 0$ to $t = T-1$ and telescoping on the left-hand side:
\[
\mathbb{E}\bigl[\cL(W_T)\bigr]\leq\mathbb{E}\bigl[\cL(W_0)\bigr] -\sum_{t=0}^{T-1}\left(\eta-\frac{(m+c+\cM)L\eta^2}{2\cM}\right)
\mathbb{E}\bigl[\|G_t\|^2 \bigr].
\]
Rearrange to isolate the sum of gradient norms:
\[
  \sum_{t=0}^{T-1}
    \mathbb{E}\bigl[\|G_t\|^2\bigr]
  \leq
  \frac{\cL(W_0) - \mathbb{E}\bigl[\cL(W_T)\bigr]}{
      \eta
      - \frac{(m+c+\cM)L\eta^2}{2\cM}
    }\leq \frac{\cL(W_0) - \cL(W^*)}{
      \eta
      - \frac{(m+c+\cM)L\eta^2}{2\cM}
    },
\]
where $W^*$ is a minimizer for $\cL$.
Choosing $\eta = \frac{\cM}{(m+c+\cM)L}:=C$ yields
\[
  \frac{1}{T}\sum_{t=0}^{T-1}
\mathbb{E}\bigl[\|G_t\|^2\bigr]
  \leq
  \frac{2(m+c+\cM)L}{\cM T}
  \bigl(\cL(W_0) - \cL(W^*)\bigr)=\frac{2C}{T}
  \bigl(\cL(W_0) - \cL(W^*)\bigr).
\]
Thus, on average, the norm of the true gradient converges to zero at a rate $O(1/T)$.
\end{proof}



\subsection{Proof of Theorem~\ref{thm:convergence}}\label{appendix:projfactor-convergence}
To analyze the convergence properties of ProjFactor, we leverage the Hamiltonian descent framework~\cite{maddison2018hamiltonian, chen2023lion, liang2024memory, nguyen2024h}, a powerful tool for understanding the behavior of optimizers in continuous time. This framework allows us to model ProjFactor’s update rule as an ordinary differential equation (ODE), providing insights into its long-term stability and convergence.

The infinitesimal updates of Projfactor is defined as follows:
\begin{equation}\label{eq:projfactor_infinitesimal_appendix}
\begin{aligned}
&\frac{d}{dt} \tilde{m}^s_t = a(\tilde{G}^s_t-\tilde{m}^s_t); \;\hat{v}^o_t = \frac{\tilde{v}^o_{rt}\tilde{v}^o_{ct}}{\mathbf{1}_n^T \tilde{v}^o_{rt}};\\
&\frac{d}{dt}\tilde{v}^o_{rt} = b((\tilde{G}_t^o)^{\odot 2}\mathbf{1}_m-\tilde{v}^o_{rt});\\
&\frac{d}{dt}\tilde{v}^o_{ct} = b(\mathbf{1}_n^T (\tilde{G}_t^o)^{\odot 2} - \tilde{v}^o_{ct});\\
&\frac{d}{dt}W_t = \operatorname{Reshape}\left(-\tilde{m}^s_t \tilde{P}^\top \Big /\sqrt{\hat{v}^o_t},\;[n,m]\right)
\end{aligned}
\end{equation}
The corresponding Lyapunov function (Hamiltonian) is defined as 
\begin{equation*}
\cH(W,\tilde{m}^s,\tilde{v}^o_{r},\tilde{v}^o_c) = \cL(W)+ \frac{1}{2a}\left\langle \tilde{m}^s, \frac{\tilde{m}^s}{\sqrt{\hat{v}^o}}\right\rangle.
\end{equation*}
Subsequently, we make the following mild assumptions, consistent with prior works in this area, to establish the mathematical foundation for our analysis. \cite{maddison2018hamiltonian,chen2023lion,liang2024memory,nguyen2024h}.
\begin{assumption}\label{assump:main-thm_appendix}
Assume the functions in system \eqref{eq:projfactor_infinitesimal_appendix} are continuously differentiable, and
\begin{enumerate}[(1)]
    \item $\frac{d}{dt}\cH(W_t,\tilde{m}_t^s,\tilde{v}^o_{rt},\tilde{v}^o_{ct})=0$ implies $\tilde{G}^s_t=0$.
    \item For any $t>0$, if $G_t\neq 0$, then $\tilde{G}_t^s \neq 0$ and $\tilde{G}_t^o \neq 0$. 
    \item For any $t>0$, $\frac{\|\tilde{G}_t^o\|^2}{\|\tilde{v}^o_{rt}\|}\leq R$.
\end{enumerate}
\end{assumption}
Here, Assumption~\ref{assump:main-thm_appendix} (1) claims that the system's Lyapunov function reaches a stationary point only when  $\tilde{G}^s_t =0$.  This condition aligns with the behavior of widely used optimizers like SGD with momentum and Adam~\cite{liang2024memory}. 
Assumption~\ref{assump:main-thm_appendix} (2) prevents the projection operator $\tilde{P}$
from annihilating nonzero gradients. Whenever $\tilde{G}^s=0$ or $\tilde{G}^o=0$, we have $G=0$, which maintains consistency between projected and original spaces. Assumption~\ref{assump:main-thm_appendix} (3) imposes a reasonable bound on the ratio of the squared gradient norm to the second moment. This bound can be intuitively derived by expanding the second-moment update rule in ProjFactor (Algorithm~\ref{ag:projfactor}):
$$\tilde{v}^{o}_{r} \gets \beta_2 \tilde{v}^{o}_{r} + (1 - \beta_2)\left(\tilde{G}^s_{t}\tilde{P}^{\top}\right)^{\odot 2} \bm{1}_{m}.$$ By $\tilde{G}^o_t=\tilde{G}^s_{t}\tilde{P}^{\top}$ and the summation of geometric series, we further have 
\[\tilde{v}^{o}_{rt} = \sum_{\tau=1}^{t} \beta_2^{t - \tau} (1 - \beta_2) (\tilde{G}^o_\tau)^{\odot 2}\mathbf{1}_m. \]
Hence, if $\tilde{G}^o_t\to 0$ as $t\to \infty$, $\frac{\|\tilde{G}_t^o\|^2}{\|\tilde{v}^o_{rt}\|}$ will be bounded.

Now we present the following convergence analysis, which interprets Projfactor from the perspective of the Hamiltonian descent method. 
\begin{theorem}\label{thm:convergence_projfactor_appendix}
Suppose the functions in system \eqref{eq:projfactor_infinitesimal_appendix} are continuously differentiable. Under Assumption~\ref{assump:main-thm_appendix}, we have
\begin{enumerate}
    \item For $(W_t, \tilde{m}^s_t, \tilde{v}^o_{rt}, \tilde{v}^o_{ct})$ satisfying \eqref{eq:projfactor_infinitesimal_appendix}, $$\frac{d}{dt}\cH(W_t, \tilde{m}^s_t, \tilde{v}^o_{rt}, \tilde{v}^o_{ct})\leq 0.$$
    \item Any bounded solution \((W_t, \tilde{m}^s_t, \tilde{v}^o_{rt}, \tilde{v}^o_{ct})_t\) of \eqref{eq:projfactor_infinitesimal_appendix} converges to a stationary point of \(\mathcal{L}(W)\) as \(t \to \infty\).
\end{enumerate}
\end{theorem}
\begin{proof}
First, we prove that $\frac{d}{dt}\cH(W_t, \tilde{m}_t^s, \tilde{v}^o_{rt}, \tilde{v}^o_{ct})\leq 0$. For simplicity, we denote that $$R_t := \frac{1}{2a}\left\langle \tilde{m}^s, \frac{\tilde{m}^s}{\sqrt{\hat{v}^o}}\right\rangle.$$
By chain rule of derivatives, we have
\begin{equation}\label{eq:Hamilton-derivative-expansion}
\begin{aligned}
\frac{d}{dt}\cH(W_t, \tilde{m}_t^s, \tilde{v}^o_{rt}, \tilde{v}^o_{ct}) &= \frac{d}{dt}\cL(W_t) + \frac{d}{dt}R_t\\
& = \left\langle\frac{d\cL(W_t)}{dW_t}, \frac{dW_t}{dt}\right\rangle + \left\langle\frac{dR_t}{d\tilde{m}_t^s}, \frac{d\tilde{m}_t^s}{dt}\right\rangle 
+ \left\langle\frac{dR_t}{d\tilde{v}^o_{rt}}, \frac{d\tilde{v}^o_{rt}}{dt}\right\rangle 
+ \left\langle\frac{dR_t}{d\tilde{v}^o_{ct}}, \frac{d\tilde{v}^o_{ct}}{dt}\right\rangle.
\end{aligned}
\end{equation}
We compute the terms in \eqref{eq:Hamilton-derivative-expansion} respectively. Firstly, by the dynamics in \eqref{eq:projfactor_infinitesimal_appendix}, 
\begin{equation}\label{eq:dMtdt_appendix}
\begin{aligned}
\left\langle\frac{d\cL(W_t)}{dW_t}, \frac{dW_t}{dt}\right\rangle + \left\langle\frac{dR_t}{d\tilde{m}_t^s}, \frac{d\tilde{m}_t^s}{dt}\right\rangle  
&= \left\langle \nabla \cL(W_t), -\operatorname{Reshape}\left(\frac{\tilde{m}_t^s \tilde{P}^{\top}}{\sqrt{\hat{v}^o_t}}, [n,m]\right)\right\rangle + \left\langle \frac{1}{2a}\frac{2\tilde{m}_t^s\sqrt{\mathbf{1}_n^{\top}\tilde{v}^o_{rt}}}{\sqrt{\tilde{v}^o_{rt}\tilde{v}^o_{ct}}}, a(\tilde{G}^s_t-\tilde{m}_t^s)\right\rangle\\
&= -\operatorname{tr}\left(\frac{\tilde{m}_t^s \tilde{P}^{\top}\operatorname{Reshape}\left(\nabla\cL(W_t)^{\top}, [nc,\frac{m}{c}]\right)}{\sqrt{\hat{v}^o_t}}\right) + \operatorname{tr}\left(\frac{\tilde{m}_t^s (\tilde{G}_t^s)^{\top}}{\sqrt{\hat{v}^o_t}}\right) - \left\langle\frac{\tilde{m}_t^s}{\sqrt{\hat{v}^o_t}}, \tilde{m}_t^s\right\rangle\\
&= - \left\langle\frac{\tilde{m}_t^s}{\sqrt{\hat{v}^o_t}}, \tilde{m}_t^s\right\rangle,
\end{aligned}
\end{equation}
where the third equality is based on the fact that reshaping both matrices of Frobenius inner product does not change the result. 

Secondly, notice that $\left\langle\frac{dR_t}{d\tilde{v}^o_{rt}}, \frac{d\tilde{v}^o_{rt}}{dt}\right\rangle$ is the inner product of two $nc\times 1$ vectors, we assume $(\tilde{v}^o_{rt})_k$ to be the $k$-th element of $\tilde{v}^o_{rt}$, and $(\tilde{G}_t^o)_{k,:}$ to be the $k$-th row of $\tilde{G}_t^o$. Recalling the ODE dynamics of $\tilde{v}^o_{rt}$ 
 in \eqref{eq:projfactor_infinitesimal_appendix}, we further have
 {\footnotesize
\begin{equation}\label{eq:dvrdt_appendix}
\begin{aligned}
\left\langle\frac{dR_t}{d\tilde{v}^o_{rt}}, \frac{d\tilde{v}^o_{rt}}{dt}\right\rangle 
&= \left(\frac{dR_t}{d\tilde{v}^o_{rt}}\right)^{\top} \frac{d\tilde{v}^o_{rt}}{dt} = \sum_{k=1}^n \frac{dR_t}{d(\tilde{v}^o_{rt})_k} \frac{d(\tilde{v}^o_{rt})_k}{dt}\\
&= \sum_{k=1}^n \left(\left(\frac{1}{2a}\sum_{i=1}^n \sum_{j=1}^m\frac{(\tilde{m}_t^s)_{ij}^2}{2\sqrt{(\tilde{v}^o_{rt})_i (\tilde{v}^o_{ct})_j}\sqrt{\mathbf{1}_n^{\top}\tilde{v}^o_{rt}}}
- \frac{1}{2a}\sum_{j=1}^m\frac{(\tilde{m}_t^s)_{kj}^2\sqrt{\mathbf{1}_n^{\top}\tilde{v}^o_{rt}}}{2\sqrt{(\tilde{v}^o_{rt})^3_k (\tilde{v}^o_{ct})_j}}
\right)\cdot b\left((\tilde{G}_t^o)^{\odot 2}_{k,:}\mathbf{1}_m-(\tilde{v}^o_{rt})_k \right)\right)\\
& = \frac{b}{4a}\sum_{k=1}^n\left(\left\langle \tilde{m}_t^s, \frac{\tilde{m}_t^s}{\sqrt{\hat{v}^o_t}\mathbf{1}_n^{\top}\tilde{v}^o_{rt}}\right\rangle \left((\tilde{G}_t^o)^{\odot 2}_{k,:}\mathbf{1}_m-(\tilde{v}^o_{rt})_k \right)\right)-\frac{b}{4a}\sum_{k=1}^n \sum_{j=1}^m\frac{(\tilde{m}_t^s)_{kj}^2\sqrt{\mathbf{1}_n^{\top}\tilde{v}^o_{rt}}(\tilde{G}_t^o)^{\odot 2}_{k,:}\mathbf{1}_m}{\sqrt{(\tilde{v}^o_{rt})^3_k (\tilde{v}^o_{ct})_j}}\\
& \qquad
+\frac{b}{4a}\left\langle \tilde{m}_t^s, \frac{\tilde{m}_t^s}{\sqrt{\hat{v}^o_t}}\right\rangle\\
& = \frac{b}{4a}\left\langle \tilde{m}_t^s, \frac{\tilde{m}_t^s}{\sqrt{\hat{v}^o_t}\mathbf{1}_n^{\top}\tilde{v}^o_{rt}}\right\rangle \left(\|\tilde{G}_t^o\|^2-\mathbf{1}_n^{\top}\tilde{v}^o_{rt} \right) -\frac{b}{4a}\sum_{k=1}^n \sum_{j=1}^m\frac{(\tilde{m}_t^s)_{kj}^2\sqrt{\mathbf{1}_n^{\top}\tilde{v}^o_{rt}}(\tilde{G}_t^o)^{\odot 2}_{k,:}\mathbf{1}_m}{\sqrt{(\tilde{v}^o_{rt})_k (\tilde{v}^o_{ct})_j}\mathbf{1}_n^{\top}\tilde{v}^o_{rt}}+\frac{b}{4a}\left\langle \tilde{m}_t^s, \frac{\tilde{m}_t^s}{\sqrt{\hat{v}^o_t}}\right\rangle\\
& =  \frac{b}{4a} \frac{1}{\mathbf{1}_n^{\top}\tilde{v}^o_{rt}} \left(\left\langle \tilde{m}_t^s, \frac{\tilde{m}_t^s}{\sqrt{\hat{v}^o_t}}\right\rangle \|\tilde{G}_t^o\|^2-\sum_{k=1}^n \sum_{j=1}^m\frac{(\tilde{m}_t^s)_{kj}^2\sqrt{\mathbf{1}_n^{\top}\tilde{v}^o_{rt}}(\tilde{G}_t^o)^{\odot 2}_{k,:}\mathbf{1}_m}{\sqrt{(\tilde{v}^o_{rt})_k (\tilde{v}^o_{ct})_j}}\right)\\
& \leq \frac{bR}{4a}\left\langle \tilde{m}_t^s, \frac{\tilde{m}_t^s}{\sqrt{\hat{v}^o_t}}\right\rangle,
\end{aligned}
\end{equation}}
where the inequality comes from Assumption~\ref{assump:main-thm_appendix} (3) and the fact that $-\sum_{k=1}^n \sum_{j=1}^m\frac{(\tilde{m}_t^s)_{kj}^2\sqrt{\mathbf{1}_n^{\top}\tilde{v}^o_{rt}}(\tilde{G}_t^o)^{\odot 2}_{k,:}\mathbf{1}_m}{\sqrt{(\tilde{v}^o_{rt})_k (\tilde{v}^o_{ct})_j}}\leq 0$.

Thirdly, since $\left\langle\frac{dR_t}{d\tilde{v}^o_{ct}}, \frac{d\tilde{v}^o_{ct}}{dt}\right\rangle$ is the inner product of two $1\times m$ vectors, we assume $(\tilde{v}^o_{ct})_l$ to be the $l$-th element of $\tilde{v}^o_{ct}$, and $(\tilde{G}_t^o)_{:,l}$ to be the $l$-th column of $\tilde{G}_t^o$. With the ODE dynamics of $\tilde{v}^o_{ct}$ in \eqref{eq:projfactor_infinitesimal_appendix}, we deduce
\begin{equation}\label{eq:dvcdt_appendix}
\begin{aligned}
\left\langle\frac{dR_t}{d\tilde{v}^o_{ct}}, \frac{d\tilde{v}^o_{ct}}{dt}\right\rangle &=
\frac{dR_t}{d\tilde{v}^o_{ct}} \left(\frac{d\tilde{v}^o_{ct}}{dt}\right)^{\top} = \sum_{l=1}^m \frac{dR_t}{d(\tilde{v}^o_{ct})_l} \left(\frac{d(\tilde{v}^o_{ct})_l}{dt}\right)\\
&= \sum_{l=1}^m  \left(\sum_{i=1}^n\left(\frac{1}{2a} \frac{-(\tilde{m}_t^s)^2_{il}\sqrt{\mathbf{1}_n^{\top}\tilde{v}^o_{rt}}}{2\sqrt{(\tilde{v}^o_{rt})_i(\tilde{v}^o_{ct})^3_l}}\right) \cdot 
b\left(\mathbf{1}_n^{\top}(\tilde{G}_t^o)^{\odot 2}_{:,l}-(\tilde{v}^o_{ct})_l\right)\right)\\
&=\frac{b}{4a}\sum_{l=1}^m  \sum_{i=1}^n\left( -\frac{(\tilde{m}_t^s)^2_{il}\mathbf{1}_n^{\top}(\tilde{G}_t^o)^{\odot 2}_{:,l}}{\sqrt{(\tilde{v}^o_{rt})_i(\tilde{v}^o_{ct})^3_l}} + \frac{(\tilde{m}_t^s)^2_{il}}{\sqrt{(\tilde{v}^o_{rt})_i(\tilde{v}^o_{ct})_l}}\right)\sqrt{\mathbf{1}_n^{\top}\tilde{v}^o_{rt}}\\
&\leq \frac{b}{4a}\sum_{l=1}^m  \sum_{i=1}^n\left(\frac{(\tilde{m}_t^s)^2_{il}}{\sqrt{(\tilde{v}^o_{rt})_i(\tilde{v}^o_{ct})_l}}\right)\sqrt{\mathbf{1}_n^{\top}\tilde{v}^o_{rt}}\\
&= \frac{b}{4a} \left\langle\frac{\tilde{m}_t^s}{\sqrt{\hat{v}^o_t}}, \tilde{m}_t^s\right\rangle,
\end{aligned}
\end{equation}
where the inequality is due to $\frac{b}{4a}\sum_{l=1}^m  \sum_{i=1}^n -\frac{(\tilde{m}_t^s)^2_{il}\mathbf{1}_n^{\top}(\tilde{G}_t^o)^{\odot 2}_{:,l}}{\sqrt{(\tilde{v}^o_{rt})_i(\tilde{v}^o_{ct})^3_l}}\leq 0$.

Combining 
\eqref{eq:Hamilton-derivative-expansion}, \eqref{eq:dMtdt_appendix}, \eqref{eq:dvrdt_appendix}, \eqref{eq:dvcdt_appendix}, and setting $a\geq (R+1)b/4a$, we have 
\begin{equation}
    \frac{d}{dt}\cH(W_t, \tilde{m}_t^s, \tilde{v}^o_{rt}, \tilde{v}^o_{ct}) \leq -\left(1-\frac{(R+1)b}{4a}\right) \left\langle\frac{\tilde{m}_t^s}{\sqrt{\hat{v}^o_t}}, \tilde{m}_t^s\right\rangle \leq 0, 
\end{equation}
which demonstrates that the Lyapunov function $\cH(W_t, \tilde{m}_t^s, \tilde{v}^o_{rt}, \tilde{v}^o_{ct})$ is monotonically decreasing along the ODE trajectory.

Furthermore, define \[
\mathcal{I} = \left\{\text{the union of complete trajectories satisfying } \frac{d}{dt} \cH(W_t, \tilde{m}_t^s, \tilde{v}^o_{rt}, \tilde{v}^o_{ct}) = 0, \, \forall t \right\}.
\]

Subsequently, by LaSalle’s invariance principle \cite{lasalle1960some}, as $t\to +\infty$, the accumulation points of any trajectory $\{(W_t, \tilde{m}_t^s, \tilde{v}^o_{rt}, \tilde{v}^o_{ct})\}_t$ lies in $\cI$. By Assumption~\ref{assump:main-thm_appendix} (1), the points in the limit set $\mathcal{I}$ should satisfy that for any $t$, $\tilde{G}^s_t =0$.
Hence, by Assumption~\ref{assump:main-thm_appendix} (2), we further have $G_t = \nabla\cL(W_t)=0$ for any $t$. This indicates that any trajectory will converge to the local optimum.
\end{proof}

Theorem~\ref{thm:convergence_projfactor_appendix} demonstrates the Lyapunov function's monotonic descent and the infinitesimal ODE system's convergence to local optimum. These results imply that ProjFactor stabilizes at a local optimum, provided the step sizes are sufficiently small.

\section{More Empirical Analysis}
\label{appendix:empirical_analysis}
In this section, we present additional empirical results. Specifically, we provide descriptions of the baseline methods used for comparison in \cref{appendix:baseline_description}. The loss curves corresponding to different projection granularities are reported in \cref{appendix_loss_curves_of_diff_proj_granularities}. Further details on the numerical error testing procedure are elaborated in \cref{appendix:numerical_error_testing}. In \cref{appendix:diff_ways_of_proj_matrix_gen}, we explore alternative approaches for generating projection matrices. The impact of warm-up steps on our optimization scheme is examined in \cref{appendix:study_on_the_warmup_steps}, while the effect of projection matrix update frequency is analyzed in \cref{appendix:study_on_the_Update_Frequency}. Additionally, we evaluate the throughput efficiency of different methods in \cref{appendix:throughput_analysis}. Further experimental results are provided in \cref{appendix:exp_with_diff_models} and \cref{appendix:Specific_Statistics}. A detailed overview of the hyperparameters used in our experiments is given in \cref{appendix:implementation_details}. Finally, sample prompts utilized during training are presented in \cref{appendix:showcase_of_Prompts}.

We use the implementation from HuggingFace for the Adam, Adafactor, and LoRA approaches while all other methods were implemented on our own by referencing their respective open-source code repositories. The complete code for our implementations and experiments will be released.

\subsection{Description of Baselines}
\label{appendix:baseline_description}
In this section, we provide a brief overview of the finetuning methods compared in our study. 
\begin{itemize}
    \item \textbf{Adam}~\cite{kingma2014adam} is an adaptive gradient-based optimization method that maintains exponentially moving averages of both the first and second moments of gradients. By adjusting step sizes for each parameter individually, Adam often converges faster and requires less finetuning of hyperparameters compared to non-adaptive methods.
    \item \textbf{Adafactor}~\cite{adafactor} generalizes the adaptive learning-rate principles of Adam but employs a factored approximation of second-order moments. This factorization reduces memory overhead, making Adafactor particularly suitable for large-scale training while preserving performance benefits similar to Adam.
    \item \textbf{LoRA}~\cite{lora} (Low-Rank Adaptation) provides an efficient approach to finetuning large language models by introducing a trainable, low-rank decomposition into selected layers. This design substantially reduces both memory usage and training time, all while maintaining competitive performance levels.
    
    \item \textbf{Galore}~\cite{Galore} is a recent optimizer that extends low-rank adaptation techniques by exploiting the low-rank structure in the update gradient rather than in the parameters themselves. Specifically, it applies singular value decomposition (SVD) to construct a projection matrix that projects the original gradient into a subspace.
    
    \item \textbf{fira}~\cite{fira} improves upon Galore by combining both the projected gradient and the residual component in the original space. Nonetheless, this design necessitates retaining the original gradient for each update step, which increases memory consumption under gradient accumulation.
    
    \item \textbf{APOLLO}~\cite{appollo} is a newly introduced, memory-efficient optimization algorithm that approximates channel-wise learning-rate scaling through an auxiliary low-rank optimizer state derived from random projections.
\end{itemize}
Besides, \textbf{FLoRA}~\cite{FLoRA} is equivalent to our methods by setting the granularity factor $c=1$, which also generates the gradient from Gaussian distribution.

\subsection{Loss Curves of Different Projection Granularities}
\label{appendix_loss_curves_of_diff_proj_granularities}
\cref{fig:loss_curves_of_diff_grains} presents the loss curves of two distinct-grained configurations of projections, Fine-Grained Projection $(c=256, r=1)$, and Ordinary-Grained Projection $(c=1, r=256)$, evaluated on the Commonsense170k dataset under a fixed Memory Budget $\mathcal{M}=256$. \cref{fig:loss_curves_of_diff_grains}(a) shows the performance when trained LLaMA2-7B with ProjFactor, where both projection configurations rapidly reduce the loss and perform comparably in the initial training phase. However, Fine-Grained Projection demonstrates a clear advantage over Ordinary-Grained Projection as the training continues. In contrast, \cref{fig:loss_curves_of_diff_grains}(b) illustrates the results obtained when LLaMA2-7B is trained with the Subspace Scheme (see \eqref{eq:ss_os_update_rules} and right of \cref{fig:subspace_os_scheme}). Here, Fine-Grained Projection consistently outperforms Ordinary-Grained Projection, even in the earlier stages of training. 

\begin{figure}[h!]
    \centering
    \includegraphics[width=0.98\linewidth]{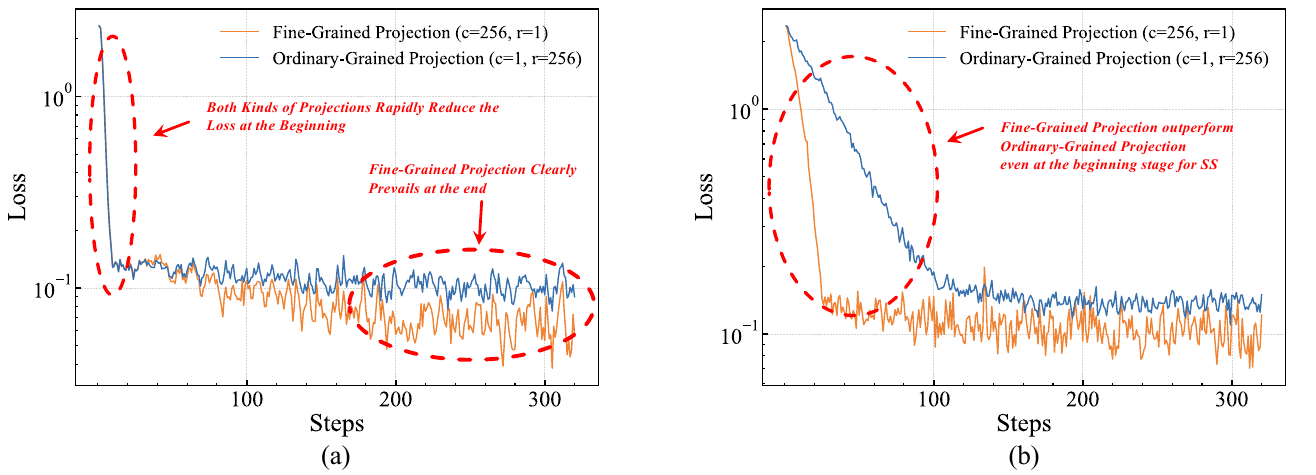} 
    \caption{Loss curves of different grained projections on the Commonsense170k dataset: (a) Training LLaMA2-7B with \textbf{ProjFactor}; (b) Training LLaMA2-7B with the Subspace Scheme (as described in \cref{sec:4}). The performance of two types of grained projections is compared under the same memory budget of $256$.}
    \label{fig:loss_curves_of_diff_grains}
\end{figure}

\subsection{Numerical Error Testing for the Projection Operation}
\label{appendix:numerical_error_testing}
\begin{figure}[h!]
    \centering
    \includegraphics[width=0.9\linewidth]{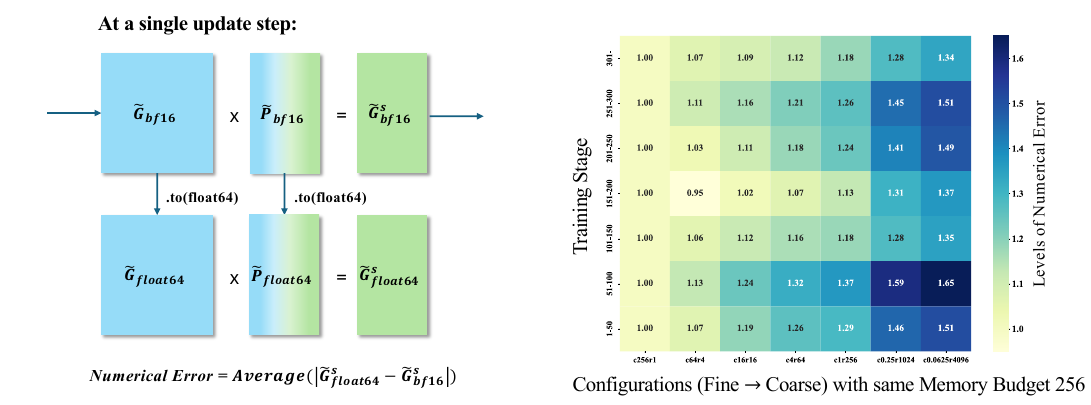} 
    \caption{\textbf{Left:} Illustration of computational numerical error for a single parameter matrix during an update step. The numerical error of the projection operator is defined as the absolute difference between $\tilde{\bm{G}}_{bf16}\tilde{P}_{bf16}$ and $\tilde{\bm{G}}_{float64}\tilde{P}_{float64}$, averaged across all parameter matrices applied low-rank gradient projection.
    \textbf{Right:} Comparison of projections' numerical errors for configurations under a constant memory budget $\mathcal{M}=256$. The y-axis denotes the training steps, which is divided into 7 stages, while the x-axis is 7 different-grained configurations.}
    \label{fig:numerical_error_exp_apx}
\end{figure}
In this section, we elaborate on the numerical error testing discussed in \cref{sec:Projections with Finer Granularity Fairly Demonstrate Superior Efficacy}. We evaluate numerical errors introduced by projection operations under varying granularities with a fixed memory budget $\mathcal{M}=256$. The procedure is shown in the LHS of \cref{fig:numerical_error_exp_apx}. Specifically, given a VLoRP configuration, at each update step, we compute $\tilde{\bm{G}}^s = \tilde{\bm{G}}\tilde{P}$ using bfloat16, denoted as $\tilde{\bm{G}}^s_{bf16} = \tilde{\bm{G}}_{bf16} \tilde{P}_{bf16}$. Simultaneously, we create double-precision copies $\tilde{\bm{G}}_{float64}$ and $\tilde{P}_{float64}$, which are numerically equivalent to their bfloat16 counterparts but retain higher precision (\textit{e.g.}, $0.1100$ vs. $0.11$). Using these, we compute $\tilde{\bm{G}}^s_{float64} = \tilde{\bm{G}}_{float64} \tilde{P}_{float64}$. The discrepancy between $\tilde{\bm{G}}^s_{float64}$ and $\tilde{\bm{G}}^s_{bf16}$ reflects numerical errors introduced by low-precision computation. For example, in low precision, $0.11 \times 0.11 = 0.01$, while in high precision, $0.1100 \times 0.1100 = 0.0121$, yielding an error of $0.0021$. The optimization is performed using the $\tilde{\bm{G}}^s_{bf16}$ datatype.

To quantify this error, we compute the absolute element-wise difference between $\tilde{\bm{G}}^s_{float64}$ and $\tilde{\bm{G}}^s_{bf16}$, averaging across all elements and parameter matrices subject to low-rank projections. This yields a numerical error metric $\delta_{(c, r)}$ for a given configuration $(c, r)$:
\begin{equation}
\delta_{(c,r)} = \operatorname{Average}_{\forall \tilde{\bm{G}}^s}\left(\operatorname{Average}_{\forall (\tilde{\bm{G}}^s)_{ij} \in \tilde{\bm{G}}^s}\left| ((\tilde{\bm{G}}^s)_{ij})_{float64} - ((\tilde{\bm{G}}^s)_{ij})_{bf16} \right|\right)
\end{equation}

The results, shown on the RHS of \cref{fig:numerical_error_exp_apx}, examine seven configurations $(c, r)$, ranging from the finest $(c=256, r=1)$ to the coarsest $(c=0.0625, r=4096)$ under the constraint $\mathcal{M}=256$. We normalize $\delta_{(c, r)}$ by $\delta_{(c=256, r=1)}$ and compute a moving average over every 50 steps. The numerical error increases almost monotonically as the configuration shifts from fine-grained $(c=256, r=1)$ to coarse-grained $(c=0.0625, r=4096)$, a trend consistent throughout training. This observation partly explains why finer-grained configurations typically yield better performance given a fixed memory budget $\mathcal{M}$---fine-grained projections have lower numerical errors introduced by using the low-precision datatype which makes the finer granularity of projection (larger $c$ albeit smaller $r$) a better choice among all the configurations shared a same $\mathcal{M}$.

\subsection{Different Ways of Projection Matrix Generation}
\label{appendix:diff_ways_of_proj_matrix_gen}
In this section, we investigate three approaches for generating the projection matrix. The first method, "normal," involves sampling the projection matrix from a standard normal distribution. The second, "Rademacher," samples the projection matrix from the Rademacher distribution, which consists of entries drawn from \(\{-1, +1\}\) with equal probability. The third approach, "SVD," constructs the projection matrix using singular value decomposition (SVD), following the methodology proposed in Galore~\cite{Galore}. These methods are evaluated on the Commonsense Reasoning task using the LLaMA2-7B model as the testbed. For all experiments, the maximum input sequence length is set to 1024, and the effective batch size is configured to 512. We finetune the model for 1 epoch in total.

\begin{figure}[h!]
    \centering
    \includegraphics[width=0.95\linewidth]{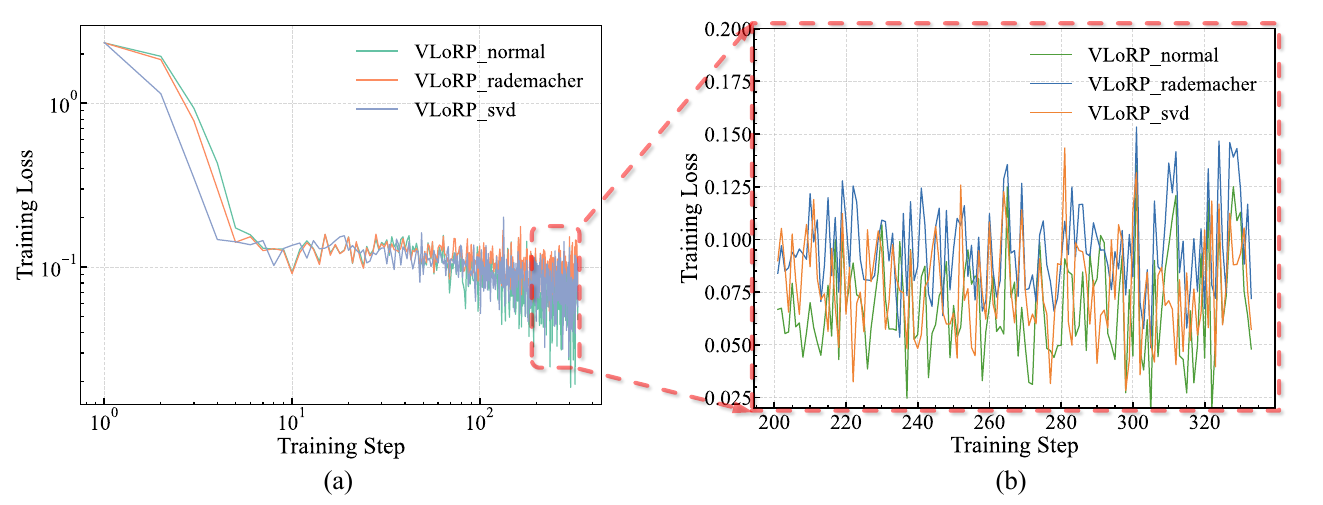} 
    \caption{Comparative Analysis of Projection Matrix Generations in LLaMA2-7B Training on the Commonsense170k Dataset: (a). Training loss curves across all training steps; (b). Zoomed-in view of convergence behavior highlighting loss variability among projection methods.}
    \label{fig:diff_ways_proj_matrix}
\end{figure}

As illustrated in \cref{fig:diff_ways_proj_matrix}(a) and the zoomed-in view in \cref{fig:diff_ways_proj_matrix}(b), the three projection matrix generation methods—normal, Rademacher, and SVD—demonstrate broadly similar training dynamics. Each approach effectively reduces the training loss during the initial steps and ultimately converges to nearly equivalent final loss values. The zoomed-in view highlights a subtle difference, with the normal-based projection slightly outperforming the others in achieving a lower training loss after convergence. These results align with the Johnson–Lindenstrauss lemma, which asserts that in high-dimensional spaces, random projections can effectively preserve geometric properties. Therefore, we adopt the sampling from normal distribution as the method for generating the projection matrix, as it provides computational and storage efficiency without compromising performance.

\subsection{Ablation Study on Update Frequency of Projection Matrix}
\label{appendix:study_on_the_Update_Frequency}
According to recent research on low-rank projection in memory-efficient LLM training~\citep{FLoRA, Galore, welore}, it is advantageous to keep the vectors \( \bm{v}_i\) constant over several training steps before resampling or reconstructing them, in order to balance the trade-off between variance and biases during training.
\begin{figure}[h!]
    \centering
    \includegraphics[width=0.95\linewidth]{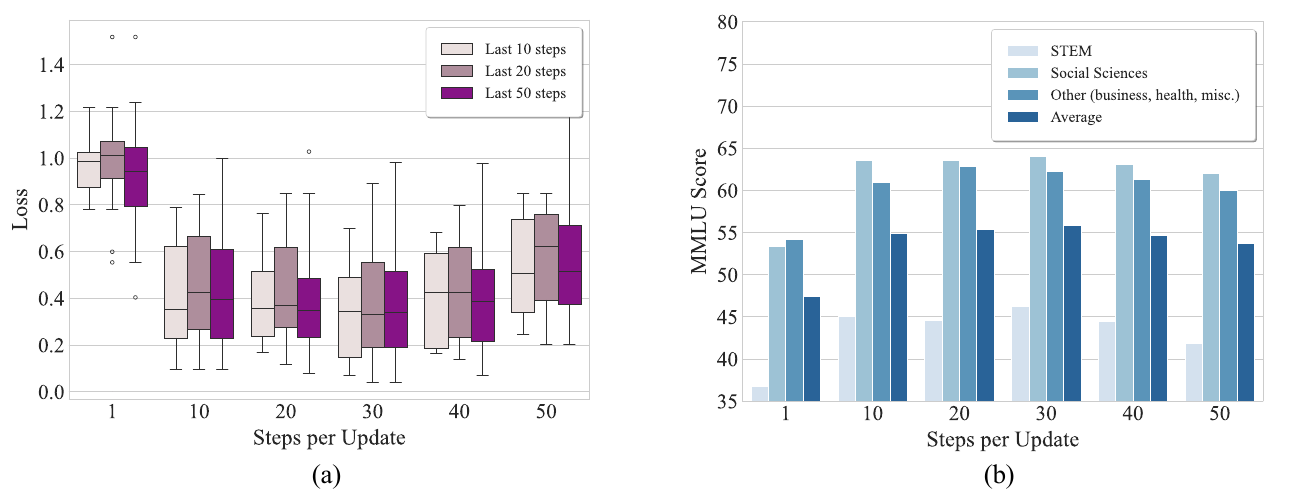} 
    \caption{Ablation study on the update frequency of the projection matrix: (a). Training loss statistics across different update frequencies of the projection matrix. "Last n steps" represents the number of final steps used to calculate the statistics; (b). MMLU scores for different categories.}
    \label{fig:update_frequency_of_proj_matrix}
\end{figure}

\cref{fig:update_frequency_of_proj_matrix} presents an ablation study evaluating the effect of varying the update frequency of the projection matrix on training loss and MMLU performance. \cref{fig:update_frequency_of_proj_matrix}(a) shows the training loss statistics calculated over the last 10, 20, and 50 steps of training for different update frequencies, while \cref{fig:update_frequency_of_proj_matrix}(b) illustrates the MMLU scores across STEM, social sciences, other (e.g., business, health), and the overall average for the same frequencies.

In \cref{fig:update_frequency_of_proj_matrix}(a), the box plots highlight that a lower update frequency, such as 1, results in higher loss values with greater variability. As the update frequency increases, the loss decreases and stabilizes, with frequencies of 20-30 demonstrating relatively consistent performance. This suggests that updating the projection matrix too frequently may introduce high variance leading to the instability of training while a larger interval for updating the projection matrix can cause the model’s updates to become constrained within fixed subspaces, leading to a degradation in performance. \cref{fig:update_frequency_of_proj_matrix}(b) reveals the impact of update frequency on MMLU scores. A similar trend emerges, where frequencies of 30 yield the highest average scores though the performance gain is slight. However, a very low-frequency 1, which updates the projection matrix at each update step results in significantly lower scores across all categories, particularly in STEM and social sciences, indicating the negative impact of high variance on the model’s ability to generalize.

\subsection{Study on the Warmup Steps}
\label{appendix:study_on_the_warmup_steps}
\cref{fig:warmup_curves} illustrates the effect of varying warm-up steps on the training loss when training our VLoRP framework with ProjFactor. The experiments evaluate five configurations: no warm-up, and warm-up steps set to 10, 20, 50, and 100. The x-axis represents the training steps on a logarithmic scale, while the y-axis shows the training loss.

\begin{figure}[h!]
    \centering
    \includegraphics[width=0.95\linewidth]{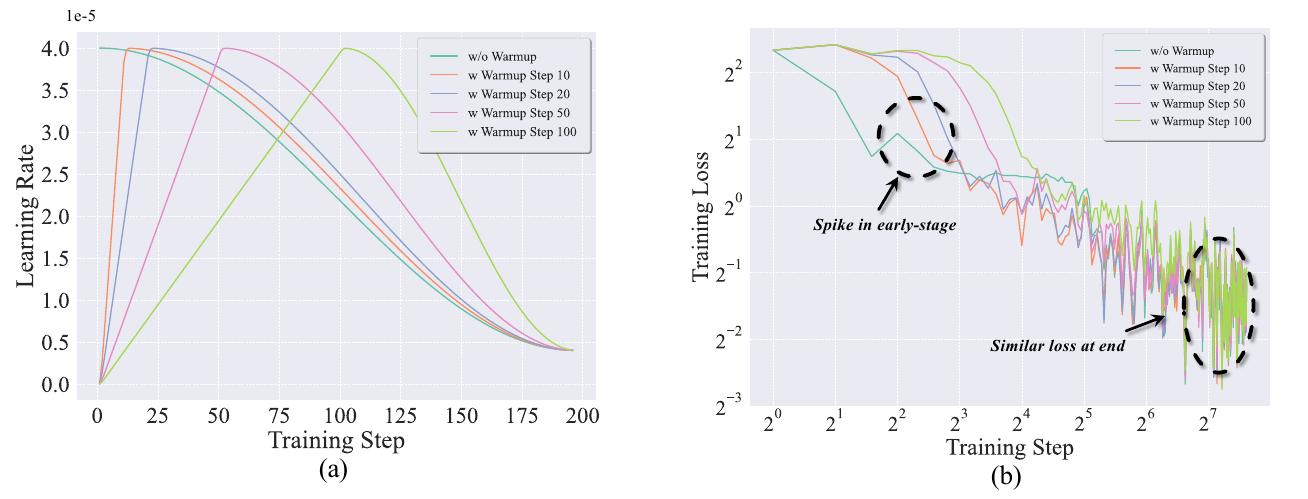} 
    \caption{Effect of different warm-up steps on training loss for VLoRP: (a). Learning rate schedules with varying warm-up steps; (b). Impact of different warm-up steps on the model's convergence.}
    \label{fig:warmup_curves}
\end{figure}

The results demonstrate that warm-up steps make the early loss curve smoother. Specifically, the configuration without warm-up exhibits a sharp spike in training loss during the initial steps, suggesting instability in optimization at the start of training. In contrast, introducing a warm-up phase helps to mitigate this instability, as evidenced by smoother loss curves for all configurations with warm-up steps. While early-stage differences are prominent, the training loss for all configurations converges to similar values by the end of training. This suggests that the choice of warm-up steps primarily impacts the transient phase of training without significantly affecting the final model performance. Notably, longer warm-up periods incur a trade-off, as they delay the convergence but enhance the stability of training in the initial phase.

\subsection{Throughput Analysis}
\label{appendix:throughput_analysis}
\begin{figure}[h!]
    \centering
    \includegraphics[width=0.7\textwidth]{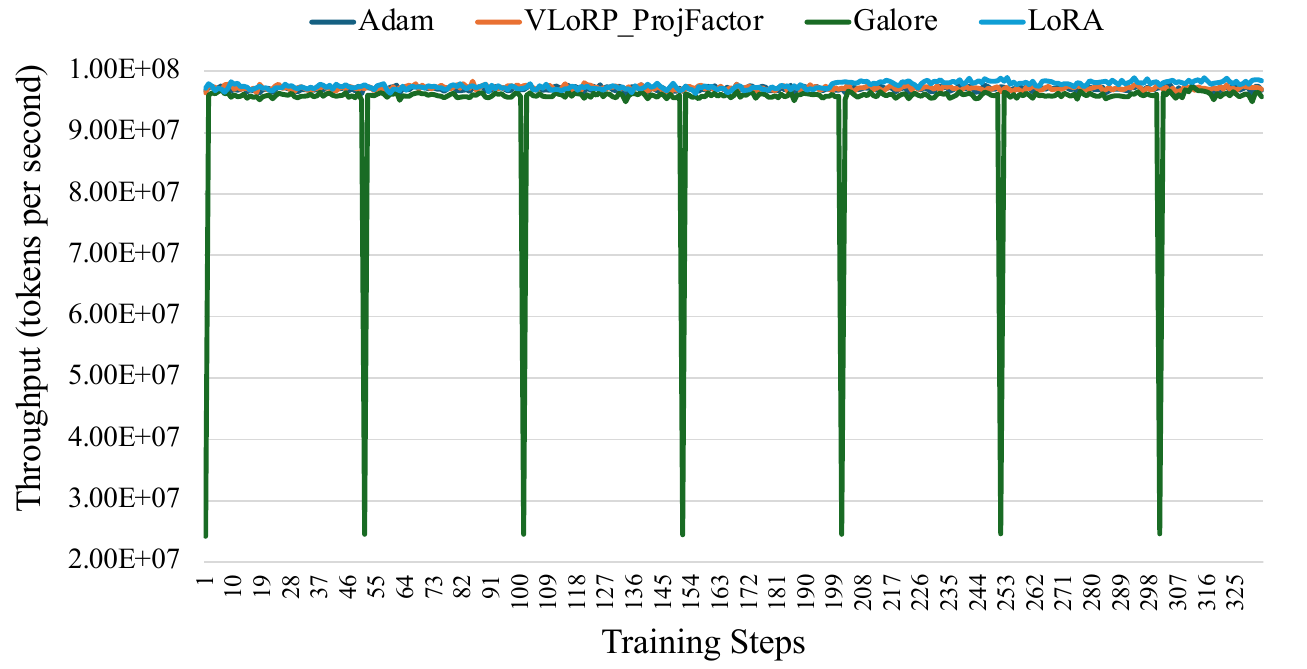} 
    \caption{Throughput Analysis of our method, LoRA, Galore, and Adam. Throughput, plotted on the y-axis, is defined as the number of tokens (including padding tokens) processed per second, while the x-axis represents the training step. Please note that the value of throughput would be influenced by factors such as the input sentence length, effective batch size, GPU hardware, and the number of padding tokens.}
    \label{fig:throughput}
\end{figure}
This section presents a comparative analysis of the throughput performance of VLoRP, LoRA, Galore, and the baseline Adam optimizer during training. For LoRA and Galore, the rank $r$ is set to $256$. In the case of VLoRP, the projection matrix is generated by sampling from a standard normal distribution, and the ProjFactor optimization algorithm is employed. For VLoRP, the granularity factor $c$ is set to $256$, and the rank $r$ is fixed at $1$. These methods are evaluated on the Commonsense Reasoning task using the LLaMA2-7B model as the testbed. All experiments are conducted with a maximum input sequence length of $1024$ and an effective batch size of $512$.

\cref{fig:throughput} illustrates the throughput of various methods over the course of 200 training steps. Overall, our method, demonstrates comparable throughput to both LoRA and the baseline Adam optimizer, while outperforming the Galore method. A notable observation is the periodic plunges in throughput experienced by Galore, occurring approximately every 50 iterations. This behavior can be attributed to the singular value decomposition (SVD) operation~\cite{Galore} used to regenerate the projection matrix in the GaLore algorithm, with the regeneration interval explicitly set to 50 iterations, which, according to our experiments, is the optimal update interval for projection matrices in GaLore.

\subsection{Experiments with Different Models}
\label{appendix:exp_with_diff_models}
Apart from LLaMA2-7B~\cite{touvron2023llama}, we also evaluated the effectiveness of our VLoRP approach on both a weaker model, GPT2-XL, and a more powerful model, LLaMA3.2-3B~\cite{dubey2024llama}. The results are presented in \cref{tab:Commonsense_Reasoning_GPT2-XL} and \cref{fig:gsm8k_LLaMA3b}.

\begin{table*}[h]
\footnotesize
\centering
\caption{Performance Comparison on Commonsense Benchmark Tasks with GPT2-XL. The table presents the results for several baseline methods and different configurations of our proposed VLoRP approach across eight commonsense reasoning tasks. All models are first finetuned on the Commonsense170k~\cite{commonsense170k} dataset and then evaluated separately on different tasks. We set the Memory Budget as 64.}
\resizebox{1.0\textwidth}{!}{%
\begin{tabular}{l|cccccccc|c}
    \toprule
    \textbf{Methods} & \textbf{ARC\_C} & \textbf{ARC\_E} & \textbf{BoolQ} & \textbf{HellaSwag} & \textbf{OBQA} & \textbf{PIQA} & \textbf{SIQA} & \textbf{winogrande}& \textbf{Avg.} \\
    \midrule
    \textbf{Adam} & 25.51 \tiny{$\pm$ 1.27} & 57.87 \tiny{$\pm$ 1.01} & 61.19 \tiny{$\pm$ 0.85} & 40.15 \tiny{$\pm$ 0.49} & 22.40 \tiny{$\pm$ 1.87} & 71.22 \tiny{$\pm$ 1.06} & 40.23 \tiny{$\pm$ 1.11} & 58.64 \tiny{$\pm$ 1.38} & 47.15\\
    \textbf{Adafactor} & 25.09 \tiny{$\pm$ 1.27} & 57.53 \tiny{$\pm$ 1.01} & 59.63 \tiny{$\pm$ 0.86} & 39.82 \tiny{$\pm$ 0.49} & 23.00 \tiny{$\pm$ 1.88} & 71.11 \tiny{$\pm$ 1.06} & 40.02 \tiny{$\pm$ 1.11} & 59.19 \tiny{$\pm$ 1.38} & 46.92 \\
    \midrule
    \textbf{LoRA(r=64)} & 24.83 \tiny{$\pm$ 1.26} & 57.11 \tiny{$\pm$ 1.02} & 58.59 \tiny{$\pm$ 0.86} & 39.98 \tiny{$\pm$ 0.49} & 22.80 \tiny{$\pm$ 1.88} & 70.89 \tiny{$\pm$ 1.06} & 40.28 \tiny{$\pm$ 1.11} & 59.12 \tiny{$\pm$ 1.38} & 46.70 \\
    \textbf{Galore(r=64)} & 25.09 \tiny{$\pm$ 1.27} & 57.83 \tiny{$\pm$ 1.01} & 59.08 \tiny{$\pm$ 0.86} & 40.20 \tiny{$\pm$ 0.49} & 22.80 \tiny{$\pm$ 1.88} & 71.00 \tiny{$\pm$ 1.06} & 40.48 \tiny{$\pm$ 1.11} & 57.38 \tiny{$\pm$ 1.39} & 46.73 \\
    \textbf{fira(r=64)} & 25.09 \tiny{$\pm$ 1.27} & 57.87 \tiny{$\pm$ 1.01} & 59.17 \tiny{$\pm$ 0.86} & 40.13 \tiny{$\pm$ 0.49} & 22.80 \tiny{$\pm$ 1.88} & 70.89 \tiny{$\pm$ 1.06} & 40.43 \tiny{$\pm$ 1.11} & 58.01 \tiny{$\pm$ 1.39} & 46.80 \\
    \textbf{APOLLO(r=64)} & 24.74 \tiny{$\pm$ 1.26} & 58.04 \tiny{$\pm$ 1.01} & 59.69 \tiny{$\pm$ 0.86} & 39.99 \tiny{$\pm$ 0.49} & 22.20 \tiny{$\pm$ 1.86} & 70.51 \tiny{$\pm$ 1.06} & 40.43 \tiny{$\pm$ 1.11} & 58.33 \tiny{$\pm$ 1.39} & 46.74 \\
    \midrule
    \small{\textbf{VLoRP}} & \\
    - \makecell[l]{\bm{$c=2^{-6},$}} \makecell[l]{\bm{$r=2^{12}$}} & 24.91 \tiny{$\pm$ 1.26} & 57.49 \tiny{$\pm$ 1.01} & 58.84 \tiny{$\pm$ 0.86} & 40.07 \tiny{$\pm$ 0.49} & 23.20 \tiny{$\pm$ 1.89} & 71.00 \tiny{$\pm$ 1.06} & 40.48 \tiny{$\pm$ 1.11} & 59.04 \tiny{$\pm$ 1.38} & 46.88 \\
    - \makecell[l]{\bm{$c=2^{-4},$}} \makecell[l]{\bm{$r=2^{10}$}} & 25.00 \tiny{$\pm$ 1.27} & 57.41 \tiny{$\pm$ 1.01} & 59.20 \tiny{$\pm$ 0.86} & 40.04 \tiny{$\pm$ 0.49} & 23.20 \tiny{$\pm$ 1.89} & 70.73 \tiny{$\pm$ 1.06} & 40.38 \tiny{$\pm$ 1.11} & 58.48 \tiny{$\pm$ 1.38} & 46.81 \\
    - \makecell[l]{\bm{$c=2^{-2},$}} \makecell[l]{\bm{$r=2^{8}$}} & 25.00 \tiny{$\pm$ 1.27} & 57.37 \tiny{$\pm$ 1.01} & 59.51 \tiny{$\pm$ 0.86} & 40.01 \tiny{$\pm$ 0.49} & 23.00 \tiny{$\pm$ 1.88} & 71.00 \tiny{$\pm$ 1.06} & 40.48 \tiny{$\pm$ 1.11} & 58.56 \tiny{$\pm$ 1.38} & 46.87 \\
    - \makecell[l]{\bm{$c=2^{0},$}} \makecell[l]{\bm{$r=2^{6}$}} & 25.51 \tiny{$\pm$ 1.27} & 57.45 \tiny{$\pm$ 1.01} & 59.33 \tiny{$\pm$ 0.86} & 40.18 \tiny{$\pm$ 0.49} & 23.20 \tiny{$\pm$ 1.89} & 70.67 \tiny{$\pm$ 1.06} & 40.28 \tiny{$\pm$ 1.11} & 58.25 \tiny{$\pm$ 1.39} & 46.86 \\ 
    - \makecell[l]{\bm{$c=2^{0},$}} \makecell[l]{\bm{$r=2^{8}$}} & 25.17 \tiny{$\pm$ 1.27} & 57.62 \tiny{$\pm$ 1.01} & 59.54 \tiny{$\pm$ 0.86} & 40.08 \tiny{$\pm$ 0.49} & 22.60 \tiny{$\pm$ 1.87} & 71.22 \tiny{$\pm$ 1.06} & 40.38 \tiny{$\pm$ 1.11} & 58.72 \tiny{$\pm$ 1.38} & 46.92 \\
    - \makecell[l]{\bm{$c=2^{2},$}} \makecell[l]{\bm{$r=2^{4}$}} & 25.09 \tiny{$\pm$ 1.27} & 57.62 \tiny{$\pm$ 1.01} & 59.79 \tiny{$\pm$ 0.86} & 40.08 \tiny{$\pm$ 0.49} & 22.80 \tiny{$\pm$ 1.88} & 70.89 \tiny{$\pm$ 1.06} & 40.33 \tiny{$\pm$ 1.11} & 58.96 \tiny{$\pm$ 1.38} & 46.94 \\ 
    - \makecell[l]{\bm{$c=2^{4},$}} \makecell[l]{\bm{$r=2^{2}$}} & 25.17 \tiny{$\pm$ 1.27} & 57.45 \tiny{$\pm$ 1.01} & 59.72 \tiny{$\pm$ 0.86} & 40.02 \tiny{$\pm$ 0.49} & 23.40 \tiny{$\pm$ 1.90} & 70.84 \tiny{$\pm$ 1.06} & 40.53 \tiny{$\pm$ 1.11} & 58.96 \tiny{$\pm$ 1.38} & 47.01 \\
    \rowcolor{lightblue}
    - \makecell[l]{\bm{$c=2^{6},$}} \makecell[l]{\bm{$r=2^{0}$}} & 25.51 \tiny{$\pm$ 1.27} & 57.87 \tiny{$\pm$ 1.01} & 60.67 \tiny{$\pm$ 0.85} & 40.11 \tiny{$\pm$ 0.49} & 23.00 \tiny{$\pm$ 1.88} & 71.06 \tiny{$\pm$ 1.06} & 40.23 \tiny{$\pm$ 1.11} & 58.56 \tiny{$\pm$ 1.38} & 47.13 \\
    \bottomrule
\end{tabular}
}
\label{tab:Commonsense_Reasoning_GPT2-XL}
\end{table*}

Specifically, Table~\ref{tab:Commonsense_Reasoning_GPT2-XL} shows the performance of VLoRP on GPT2-XL across eight commonsense reasoning benchmarks. VLoRP consistently outperforms or remains competitive with other PeFT methods such as LoRA, GaLore, and fira. Besides, among different configurations of VLoRP, finer-grained projections (\(c = 2^6, r = 1\)) achieve the highest average score (47.13), demonstrating the effectiveness of fine-grained low-rank projections.  

\begin{figure}[h!]
    \centering
    \includegraphics[width=0.95\linewidth]{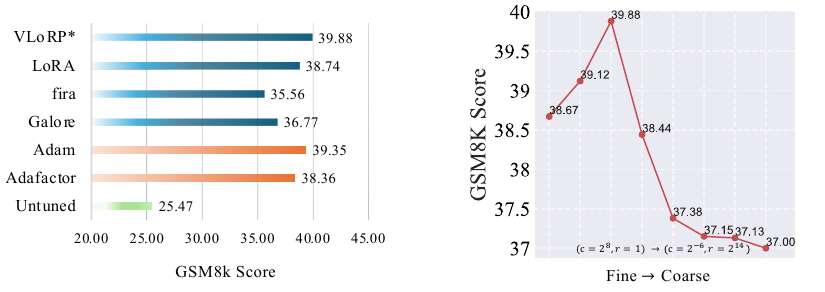} 
    \caption{\textbf{Left:} Performance comparison of different methods on GSM8K with LLaMA3.2-3B. \textbf{Right:} Performance comparison among the configurations of VLoRP with $\mathcal{M} = 256$. The x-axis indicates configurations from fine to coarse (left to right)}
    \label{fig:gsm8k_LLaMA3b}
\end{figure}

Figure~\ref{fig:gsm8k_LLaMA3b} further evaluates VLoRP on LLaMA3.2-3B using the GSM8K benchmark. The left subfigure compares VLoRP to other adaptation methods, where VLoRP achieves the highest GSM8K score (39.88). The right subfigure examines the impact of projection granularity on VLoRP’s performance. Instead of achieving the best performance at the finest granularity $(c=2^8, r=1)$, we find that the slightly coarser configuration (\(c = 2^4, r = 2^4\)) reaches the highest GSM8K score of 39.88. However, a general trend where finer-grained configurations yield better results still holds. 


\subsection{Specific Statistics of \cref{fig:ablation_of_memory_budget}}
\label{appendix:Specific_Statistics}
In this section, we present the detailed data for \cref{fig:ablation_of_memory_budget}, which displays the average performance across eight commonsense reasoning tasks. From \cref{tab:commonsense_memory_budget_256}, \cref{tab:commonsense_memory_budget_64}, and \cref{tab:commonsense_memory_budget_16}, it is evident that the finer configuration (larger $c$ and smaller $r$) consistently outperforms coarser configurations across nearly all tasks when evaluated under the same memory budget.

\begin{table*}[h!]
\footnotesize
\centering
\caption{Performance Comparison on Commonsense Benchmark Tasks (Memory Budget 256).}
\resizebox{1.0\textwidth}{!}{%
\begin{tabular}{l|cccccccc|c}
    \toprule
    \textbf{Configurations} & \textbf{ARC\_C} & \textbf{ARC\_E} & \textbf{BoolQ} & \textbf{HellaSwag} & \textbf{OBQA} & \textbf{PIQA} & \textbf{SIQA} & \textbf{winogrande}& \textbf{Avg.} \\
    \midrule
    - \makecell[l]{\bm{$c=2^{-6},$}} \makecell[l]{\bm{$r=2^{14}$}} &   42.92 \tiny{$\pm$ 1.45} & 76.22 \tiny{$\pm$ 0.87} & 79.27 \tiny{$\pm$ 0.71} & 57.53 \tiny{$\pm$ 0.49} & 32.60 \tiny{$\pm$ 2.10} & 77.91 \tiny{$\pm$ 0.97} & 46.72 \tiny{$\pm$ 1.13} & 69.85 \tiny{$\pm$ 1.29} & 60.38 \\
    - \makecell[l]{\bm{$c=2^{-4},$}} \makecell[l]{\bm{$r=2^{12}$}} &   43.34 \tiny{$\pm$ 1.45} & 76.26 \tiny{$\pm$ 0.87} & 79.54 \tiny{$\pm$ 0.71} & 57.58 \tiny{$\pm$ 0.49} & 32.00 \tiny{$\pm$ 2.09} & 77.64 \tiny{$\pm$ 0.97} & 46.72 \tiny{$\pm$ 1.13} & 70.01 \tiny{$\pm$ 1.29} & 60.39 \\
    - \makecell[l]{\bm{$c=2^{-2},$}} \makecell[l]{\bm{$r=2^{10}$}} &   43.34 \tiny{$\pm$ 1.45} & 76.30 \tiny{$\pm$ 0.81} & 79.45 \tiny{$\pm$ 0.71} & 57.47 \tiny{$\pm$ 0.49} & 32.20 \tiny{$\pm$ 2.09} & 77.75 \tiny{$\pm$ 0.97} & 46.78 \tiny{$\pm$ 1.13} & 70.01 \tiny{$\pm$ 1.29} & 60.41 \\
    - \makecell[l]{\bm{$c=2^{0},$}} \makecell[l]{\bm{$r=2^{8}$}} &   43.69 \tiny{$\pm$ 1.45} & 77.02 \tiny{$\pm$ 0.86} & 79.27 \tiny{$\pm$ 0.71} & 57.49 \tiny{$\pm$ 0.49} & 31.80 \tiny{$\pm$ 2.08} & 78.07 \tiny{$\pm$ 0.97} & 47.49 \tiny{$\pm$ 1.13} & 69.77 \tiny{$\pm$ 1.29} & 60.57 \\ 
    - \makecell[l]{\bm{$c=2^{2},$}} \makecell[l]{\bm{$r=2^{6}$}} &   44.03 \tiny{$\pm$ 1.45} & 76.81 \tiny{$\pm$ 0.87} & 79.17 \tiny{$\pm$ 0.71} & 57.59 \tiny{$\pm$ 0.49} & 31.80 \tiny{$\pm$ 2.08} & 78.02 \tiny{$\pm$ 0.97} & 47.19 \tiny{$\pm$ 1.13} & 69.53 \tiny{$\pm$ 1.29} & 60.53 \\ 
    - \makecell[l]{\bm{$c=2^{4},$}} \makecell[l]{\bm{$r=2^{4}$}} &   44.71 \tiny{$\pm$ 1.45} & 77.27 \tiny{$\pm$ 0.86} & 79.42 \tiny{$\pm$ 0.71} & 57.50 \tiny{$\pm$ 0.49} & 32.20 \tiny{$\pm$ 2.09} & 77.86 \tiny{$\pm$ 0.97} & 47.54 \tiny{$\pm$ 1.13} & 70.09 \tiny{$\pm$ 1.29} & 60.82 \\
    - \makecell[l]{\bm{$c=2^{6},$}} \makecell[l]{\bm{$r=2^{2}$}} &  44.97 \tiny{$\pm$ 1.45} & 77.65 \tiny{$\pm$ 0.85} & 80.46 \tiny{$\pm$ 0.69} & 57.56 \tiny{$\pm$ 0.49} & 33.60 \tiny{$\pm$ 2.11} & 77.97 \tiny{$\pm$ 0.97} & 48.06 \tiny{$\pm$ 1.13} & 69.69 \tiny{$\pm$ 1.29} & 61.25 \\
    \rowcolor{lightblue}
    - \makecell[l]{\bm{$c=2^{8},$}} \makecell[l]{\bm{$r=2^{0}$}} &  45.56 \tiny{$\pm$ 1.46} & 77.78 \tiny{$\pm$ 0.85} & 80.58 \tiny{$\pm$ 0.69} & 57.59 \tiny{$\pm$ 0.49} & 34.00 \tiny{$\pm$ 2.12} & 77.86 \tiny{$\pm$ 0.97} & 48.16 \tiny{$\pm$ 1.13} & 69.69 \tiny{$\pm$ 1.29} & 61.40 \\
    \bottomrule
\end{tabular}
}
\label{tab:commonsense_memory_budget_256}
\end{table*}

\begin{table*}[h!]
\footnotesize
\centering
\caption{Performance Comparison on Commonsense Benchmark Tasks (Memory Budget 64).}
\resizebox{1.0\textwidth}{!}{%
\begin{tabular}{l|cccccccc|c}
    \toprule
    \textbf{Configurations} & \textbf{ARC\_C} & \textbf{ARC\_E} & \textbf{BoolQ} & \textbf{HellaSwag} & \textbf{OBQA} & \textbf{PIQA} & \textbf{SIQA} & \textbf{winogrande}& \textbf{Avg.} \\
    \midrule
    - \makecell[l]{\bm{$c=2^{-6},$}} \makecell[l]{\bm{$r=2^{12}$}} & 42.83 \tiny{$\pm$ 1.45} & 76.09 \tiny{$\pm$ 0.87} & 78.99 \tiny{$\pm$ 0.71} & 57.54 \tiny{$\pm$ 0.49} & 31.80 \tiny{$\pm$ 2.08} & 77.86 \tiny{$\pm$ 0.97} & 46.32 \tiny{$\pm$ 1.13} & 69.77 \tiny{$\pm$ 1.29} & 60.15 \\
    - \makecell[l]{\bm{$c=2^{-4},$}} \makecell[l]{\bm{$r=2^{10}$}} & 43.00 \tiny{$\pm$ 1.45} & 76.14 \tiny{$\pm$ 0.87} & 78.99 \tiny{$\pm$ 0.71} & 57.55 \tiny{$\pm$ 0.49} & 31.80 \tiny{$\pm$ 2.08} & 77.97 \tiny{$\pm$ 0.97} & 46.16 \tiny{$\pm$ 1.13} & 69.53 \tiny{$\pm$ 1.29} & 60.14 \\
    - \makecell[l]{\bm{$c=2^{-2},$}} \makecell[l]{\bm{$r=2^{8}$}} & 43.09 \tiny{$\pm$ 1.45} & 76.30 \tiny{$\pm$ 0.87} & 78.78 \tiny{$\pm$ 0.72} & 57.57 \tiny{$\pm$ 0.49} & 31.40 \tiny{$\pm$ 2.08} & 77.97 \tiny{$\pm$ 0.97} & 46.16 \tiny{$\pm$ 1.13} & 69.06 \tiny{$\pm$ 1.30} & 60.04 \\
    - \makecell[l]{\bm{$c=2^{0},$}} \makecell[l]{\bm{$r=2^{6}$}} & 43.52 \tiny{$\pm$ 1.45} & 76.56 \tiny{$\pm$ 0.87} & 79.02 \tiny{$\pm$ 0.71} & 57.46 \tiny{$\pm$ 0.49} & 32.00 \tiny{$\pm$ 2.09} & 78.07 \tiny{$\pm$ 0.97} & 46.21 \tiny{$\pm$ 1.13} & 69.38 \tiny{$\pm$ 1.30} & 60.28 \\
    - \makecell[l]{\bm{$c=2^{2},$}} \makecell[l]{\bm{$r=2^{4}$}} & 43.43 \tiny{$\pm$ 1.45} & 76.43 \tiny{$\pm$ 0.87} & 79.11 \tiny{$\pm$ 0.71} & 57.53 \tiny{$\pm$ 0.49} & 31.40 \tiny{$\pm$ 2.08} & 78.02 \tiny{$\pm$ 0.97} & 46.26 \tiny{$\pm$ 1.13} & 69.46 \tiny{$\pm$ 1.29} & 60.21 \\
    - \makecell[l]{\bm{$c=2^{4},$}} \makecell[l]{\bm{$r=2^{2}$}} & 44.03 \tiny{$\pm$ 1.45} & 76.56 \tiny{$\pm$ 0.87} & 79.36 \tiny{$\pm$ 0.71} & 57.51 \tiny{$\pm$ 0.49} & 32.40 \tiny{$\pm$ 2.10} & 77.86 \tiny{$\pm$ 0.97} & 46.88 \tiny{$\pm$ 1.13} & 69.61 \tiny{$\pm$ 1.29} & 60.53 \\
    \rowcolor{lightblue}
    - \makecell[l]{\bm{$c=2^{6},$}} \makecell[l]{\bm{$r=2^{0}$}} & 45.39 \tiny{$\pm$ 1.45} & 77.19 \tiny{$\pm$ 0.87} & 79.91 \tiny{$\pm$ 0.71} & 57.46 \tiny{$\pm$ 0.49} & 33.20 \tiny{$\pm$ 2.11} & 77.91 \tiny{$\pm$ 0.97} & 47.70 \tiny{$\pm$ 1.13} & 69.93 \tiny{$\pm$ 1.29} & 61.09 \\
    \bottomrule
\end{tabular}
}
\label{tab:commonsense_memory_budget_64}
\end{table*}

\begin{table*}[h!]
\footnotesize
\centering
\caption{Performance Comparison on Commonsense Benchmark Tasks (Memory Budget 16).}
\resizebox{1.0\textwidth}{!}{%
\begin{tabular}{l|cccccccc|c}
    \toprule
    \textbf{Configurations} & \textbf{ARC\_C} & \textbf{ARC\_E} & \textbf{BoolQ} & \textbf{HellaSwag} & \textbf{OBQA} & \textbf{PIQA} & \textbf{SIQA} & \textbf{winogrande}& \textbf{Avg.} \\
    \midrule
    - \makecell[l]{\bm{$c=2^{-6},$}} \makecell[l]{\bm{$r=2^{10}$}} & 42.41 \tiny{$\pm$ 1.44} & 76.18 \tiny{$\pm$ 0.87} & 78.69 \tiny{$\pm$ 0.72} & 57.50 \tiny{$\pm$ 0.49} & 31.60 \tiny{$\pm$ 2.08} & 78.07 \tiny{$\pm$ 0.97} & 46.11 \tiny{$\pm$ 1.13} & 69.30 \tiny{$\pm$ 1.30} & 59.98 \\
    - \makecell[l]{\bm{$c=2^{-4},$}} \makecell[l]{\bm{$r=2^{8}$}} & 43.00 \tiny{$\pm$ 1.45} & 76.14 \tiny{$\pm$ 0.87} & 78.90 \tiny{$\pm$ 0.71} & 57.33 \tiny{$\pm$ 0.49} & 31.20 \tiny{$\pm$ 2.07} & 78.40 \tiny{$\pm$ 0.96} & 45.85 \tiny{$\pm$ 1.13} & 69.46 \tiny{$\pm$ 1.29} & 60.03 \\
    - \makecell[l]{\bm{$c=2^{-2},$}} \makecell[l]{\bm{$r=2^{6}$}} & 42.75 \tiny{$\pm$ 1.45} & 76.35 \tiny{$\pm$ 0.87} & 78.65 \tiny{$\pm$ 0.72} & 57.37 \tiny{$\pm$ 0.49} & 31.80 \tiny{$\pm$ 2.08} & 78.18 \tiny{$\pm$ 0.96} & 45.85 \tiny{$\pm$ 1.13} & 69.30 \tiny{$\pm$ 1.30} & 60.03 \\
    - \makecell[l]{\bm{$c=2^{0},$}} \makecell[l]{\bm{$r=2^{4}$}} & 43.43 \tiny{$\pm$ 1.45} & 76.22 \tiny{$\pm$ 0.87} & 79.08 \tiny{$\pm$ 0.71} & 57.56 \tiny{$\pm$ 0.49} & 32.00 \tiny{$\pm$ 2.09} & 78.02 \tiny{$\pm$ 0.97} & 46.37 \tiny{$\pm$ 1.13} & 69.69 \tiny{$\pm$ 1.29} & 60.30 \\
    - \makecell[l]{\bm{$c=2^{2},$}} \makecell[l]{\bm{$r=2^{2}$}} & 43.26 \tiny{$\pm$ 1.45} & 76.39 \tiny{$\pm$ 0.87} & 78.62 \tiny{$\pm$ 0.72} & 57.49 \tiny{$\pm$ 0.49} & 31.60 \tiny{$\pm$ 2.08} & 77.97 \tiny{$\pm$ 0.97} & 46.57 \tiny{$\pm$ 1.13} & 69.06 \tiny{$\pm$ 1.30} & 60.12 \\
    \rowcolor{lightblue}
    - \makecell[l]{\bm{$c=2^{4},$}} \makecell[l]{\bm{$r=2^{0}$}} & 44.88 \tiny{$\pm$ 1.45} & 76.94 \tiny{$\pm$ 0.86} & 79.39 \tiny{$\pm$ 0.71} & 57.63 \tiny{$\pm$ 0.49} & 33.20 \tiny{$\pm$ 2.11} & 78.07 \tiny{$\pm$ 0.97} & 47.49 \tiny{$\pm$ 1.13} & 69.46 \tiny{$\pm$ 1.29} & 60.88 \\
    \bottomrule
\end{tabular}
}
\label{tab:commonsense_memory_budget_16}
\end{table*}

\subsection{Implementation Details}
\label{appendix:implementation_details}
For all datasets, experiments were conducted on an \textbf{NVIDIA A100 GPU (80GB)} using the \textbf{bfloat16} datatype. By default, we applied our proposed ProjFactor method to VLoRP. For other low-rank-based PeFT methods, such as LoRA and GaLore, we set the rank to \( r = \mathcal{M} \) for a fair comparison, where \( \mathcal{M} \) is the memory budget defined in \cref{sec:3}.  
The training process follows a \textbf{batch size of 16} with \textbf{gradient accumulation over 32 steps}, yielding an \textbf{effective batch size of 512}. The \textbf{maximum input sequence length} is \textbf{1024 tokens}, and \textbf{activation checkpointing} is enabled to optimize memory usage. Learning rates (\(\eta\)) are set to the best values found through empirical testing: \(2\times 10^{-5}\) for Commonsense Reasoning, \(4\times 10^{-5}\) for MMLU, and \(10^{-4}\) for GSM8K. Notably, for GSM8K, we observed that Galore and fira require relatively larger learning rates, while Apollo performs badly on this benchmark.  
Regarding training duration, models are trained for \textbf{one epoch} on Commonsense Reasoning (333 iterations) and MMLU (196 iterations). For GSM8K, all models are trained for \textbf{three epochs} (138 iterations) due to: (1) the higher difficulty of GSM8K questions compared to the other benchmarks and (2) the significantly smaller training dataset size.

Besides, the reported memory usage in \cref{sec:experiments} refers to the actual GPU memory allocated once the model has stabilized, rather than the maximum memory reserved by \texttt{PyTorch}. The latter one is typically displayed by commands such as \texttt{watch nvidia-smi} and is often larger than the former counterpart.


\subsection{Showcase of Training Prompts}
\label{appendix:showcase_of_Prompts}
\includegraphics[width=\textwidth]{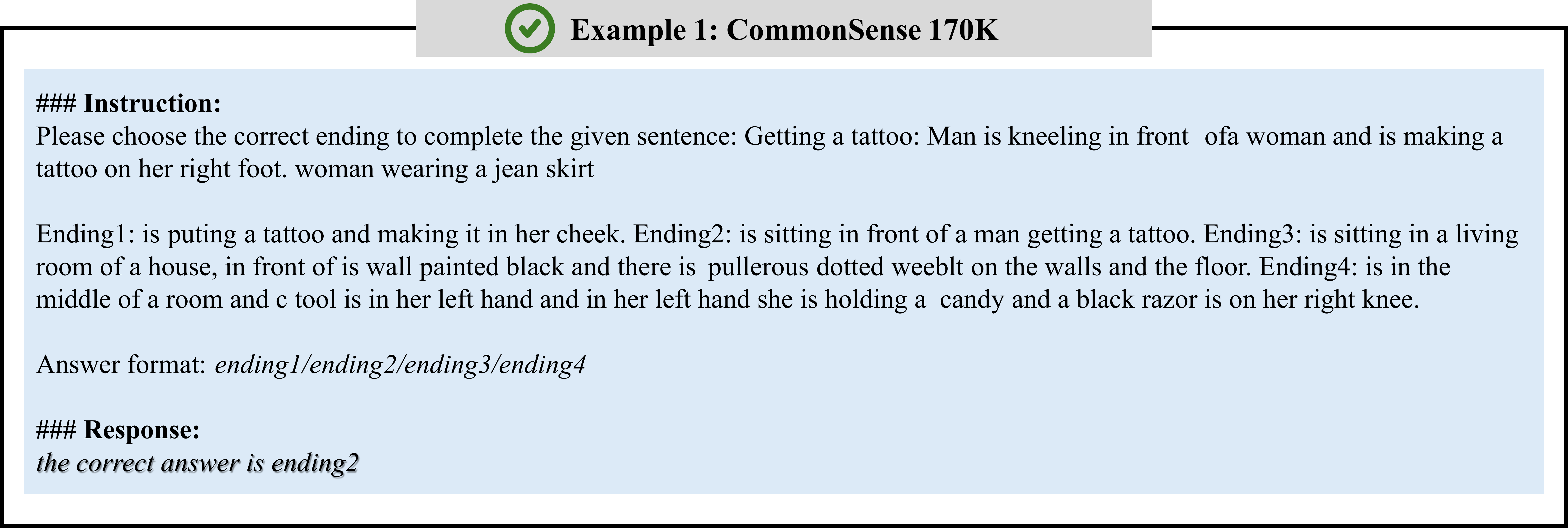}

\includegraphics[width=\textwidth]{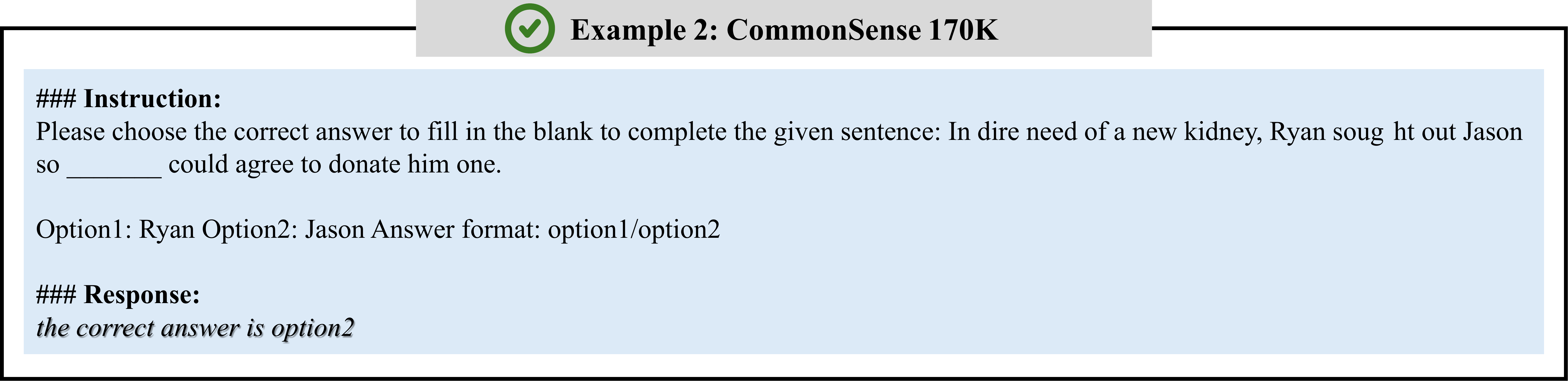}

\includegraphics[width=\textwidth]{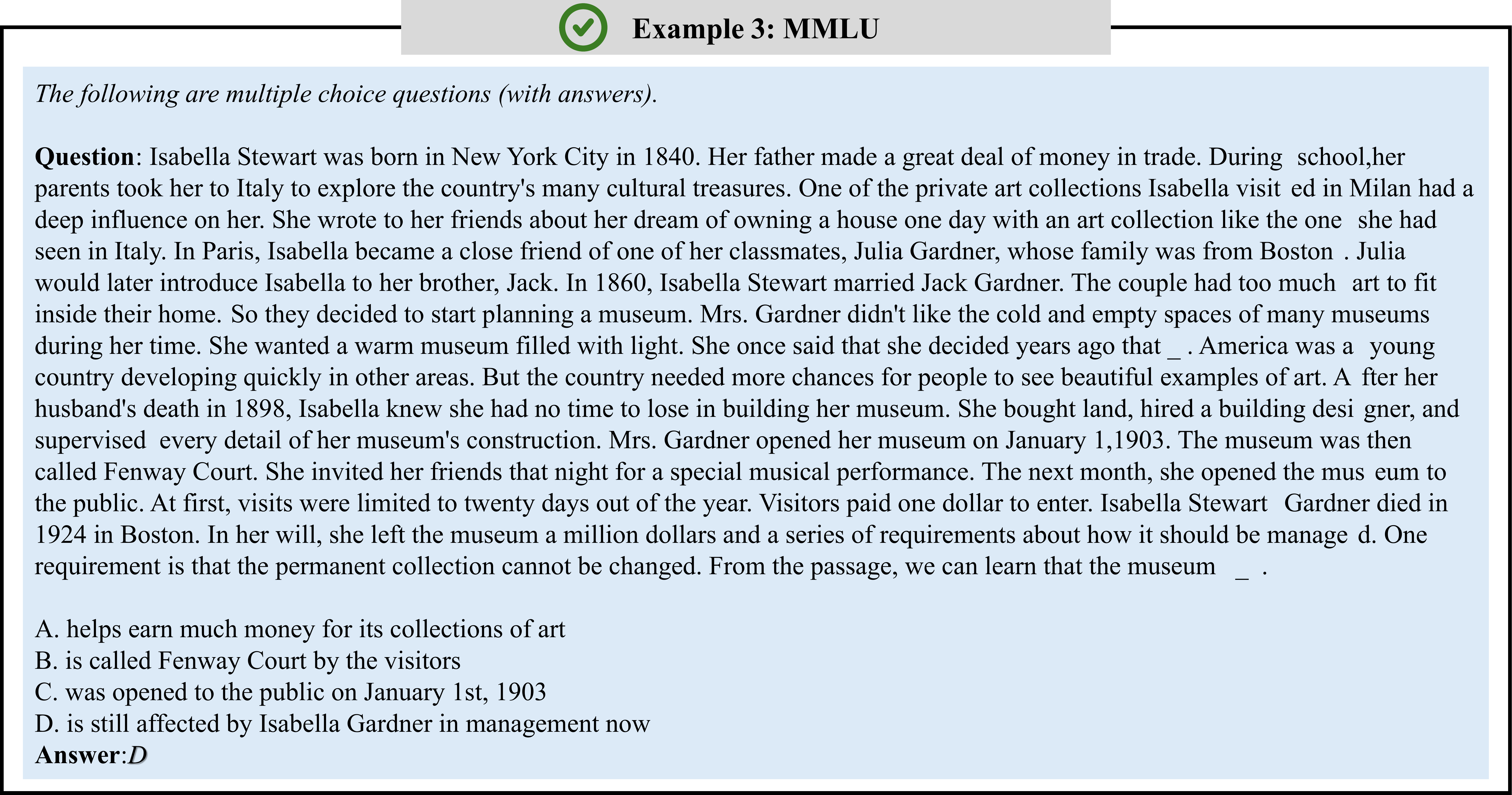}

\includegraphics[width=\textwidth]{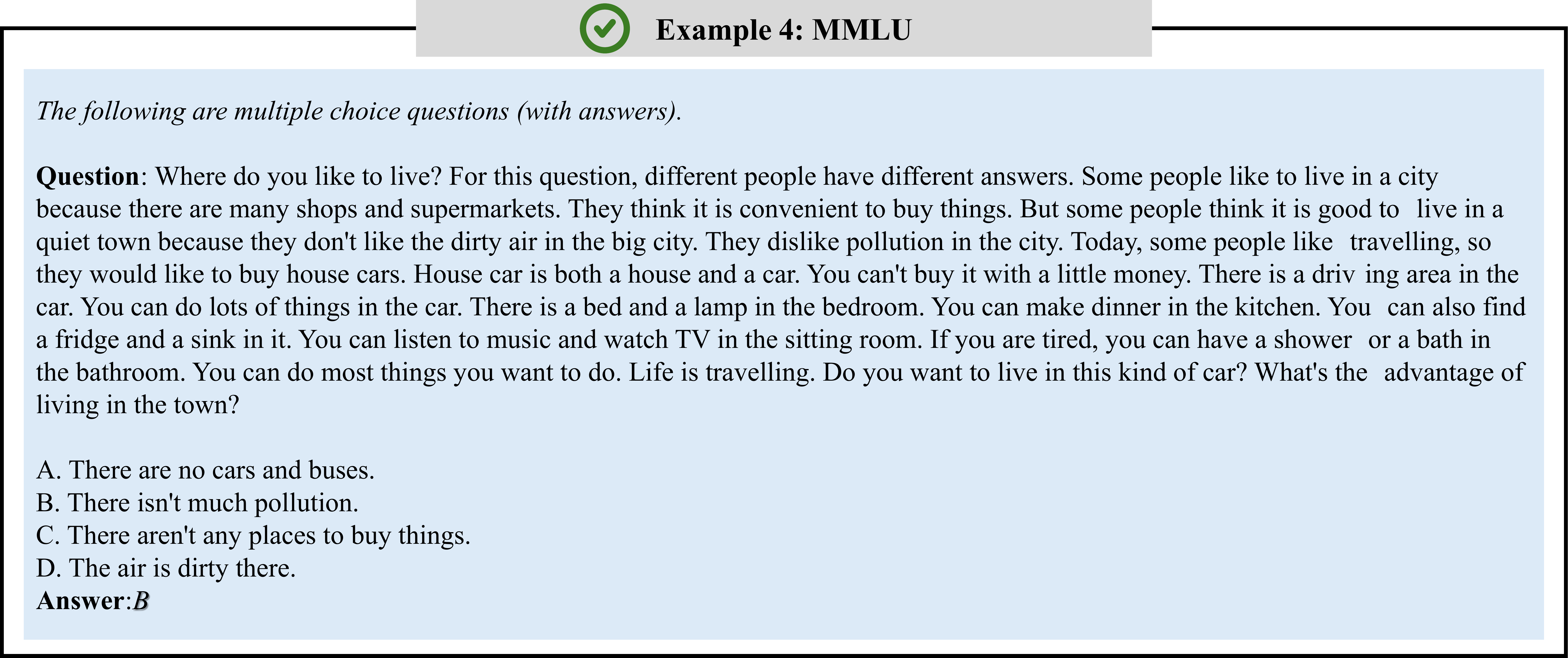}

\includegraphics[width=\textwidth]{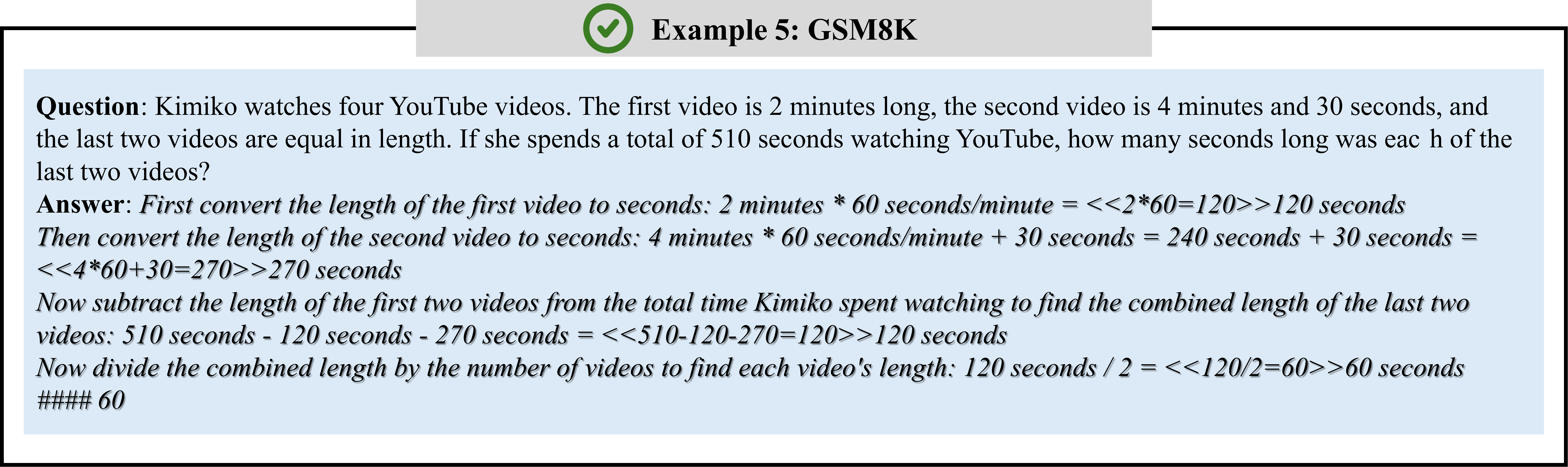}

\includegraphics[width=\textwidth]{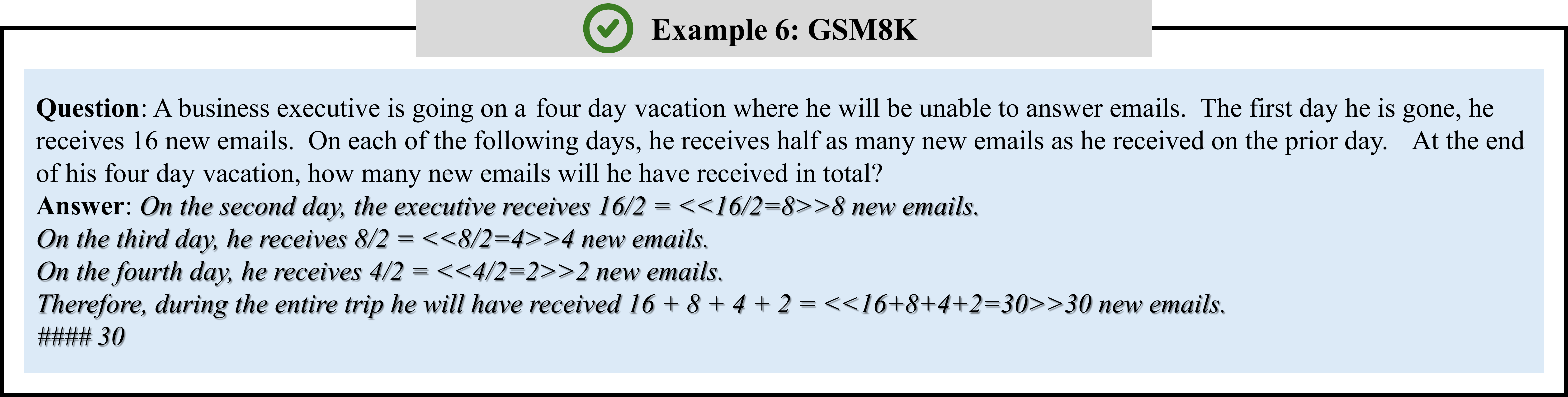}

\includegraphics[width=\textwidth]{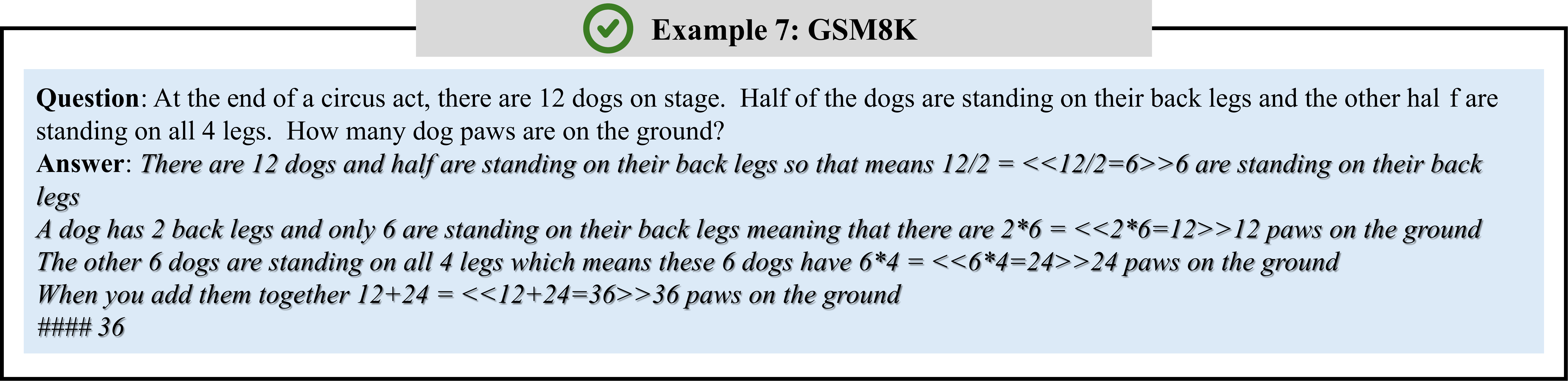}

\section{Related Works}
\label{appendix:related_works}
\subsection{Parameter-Efficient Finetuning} 
To mitigate the substantial costs associated with finetuning large-scale models, Parameter-Efficient finetuning (PEFT) methods have been introduced. These techniques adapt models to downstream tasks by training only a small fraction of the total parameters. Existing PEFT approaches can be broadly categorized into three primary families. The first category comprises adapter-based methods~\cite{Houlsby2019, He2021, Mahabadi2021}, which integrate additional trainable modules into the otherwise frozen backbone network.  For instance, \citet{Houlsby2019} proposes the sequential addition of linear modules to existing layers, while \citet{He2021} introduces the integration of these modules in parallel with the original layers to enhance performance. The second category includes prompt-based methods~\cite{Lester2021, Razdaibiedina2023, MultiLoRA}, which augment the initial input with extra soft tokens, focusing solely on finetuning these trainable vectors. However, prompt-based methods often face challenges stemming from sensitivity to initialization, which can impede their overall effectiveness. Notably, both adapter-based and prompt-based methods, whether modifying the model’s input or architecture, tend to increase inference latency compared to the baseline model. The third category is the low-rank-based methods (\textit{e.g.} LoRA) which exploit low-rank properties inside the training procedure, which we will discuss comprehensively in the next.

\subsection{Low-Rank Based Memory-Efficient Finetuning}
By using two low-rank matrices to estimate the increment of pre-trained weights without incurring additional inference overhead, LoRA~\citep{lora} and its improved variants have achieved remarkable success in the field of PeFT. For example, QLoRA~\citep{QLoRA} combines low-bit quantization with LoRA to facilitate the finetuning of LLMs. AdaLoRA~\citep{adalora} dynamically allocates the parameter budget across weight matrices based on importance scores, optimizing the use of trainable parameters. Additionally, methods such as VeRA~\citep{VeRA} reduce the number of trainable parameters by employing a single pair of low-rank matrices shared across all layers, learning small scaling vectors instead. \citet{LoRA-GA} align the gradients of the low-rank matrix product with those from full finetuning at the initial step, achieving competitive results. Furthermore, \citet{LoRA+} explore adjusting the learning rates of the LoRA adapter matrices independently, enhancing feature learning efficiency. 

Recently, another line of research~\citep{Galore, FLoRA, welore} has re-implemented LoRA methods from the perspective of low-rank gradient projection. \citet{FLoRA} investigated the training dynamics of LoRA methods and found that vanilla LoRA~\citep{lora} can be approximated by a process of randomly projecting gradients to a low-rank subspace and then projecting back. \citet{Galore} followed a similar idea but chose to obtain the projection matrix by performing Singular Value Decomposition (SVD) on the gradients to implement Galore, instead of using a randomly sampled matrix as in FloRA~\citep{FLoRA}. \citet{fira} extends Galore by incorporating the residual error between the full-rank gradient and its low-rank approximation, effectively simulating full-rank updates. APOLLO~\citep{appollo} approximates channel-wise learning-rate scaling through an auxiliary low-rank optimizer state derived from random projections.


\subsection{The Equivalence between LoRA and LoRP}
\label{sec:grad_compress_flora}
\citet{FLoRA} has shown that LoRA is approximately equivalent to an approach that compresses the gradient updates by down-projecting them onto a lower-dimensional space, and then projecting back to the original space:
\begin{theorem}[\citet{FLoRA}]
Consider a weight matrix \(W\in \mathbb{R}^{n\times m}\) with the low-rank factorization $\Delta W = BA$ where we initially have \(B_{0} = \bm{0}^{n\times r}\) and \(A_{0} \in \mathbb{R}^{r \times m}\) randomly sampled from a standard Gaussian distribution. Assuming that the learning rate \(\eta\) is sufficiently small, then the LoRA training procedure effectively restricts weight updates to the column space of \(A_{0}\).  Moreover, after \(T\) gradient-based update steps, the resulting weight matrix \(W_{T}\) satisfies
\begin{equation}
\label{eq:3.1}
W_{T} \approx W_{0} - \eta \sum_{t=0}^{T-1} G_t A_{0} A_{0}^{\top} / r,
\end{equation}
where \(G_t\) denotes the gradient of $W$ at the \(t\)-th step.
\label{thm:3.1}
\end{theorem}
The rationale of \cref{thm:3.1} is grounded in the Johnson–Lindenstrauss lemma and its extensions~\citep{jls-1, jls-2, jls-3}, which assert that random projections via Gaussian matrices approximately preserve the geometry with high probability. Moreover, \citet{FLoRA} quantifies the reconstruction error, demonstrating that the rank $r$ needs only to scale logarithmically to maintain low element-wise error, thus ensuring computational and memory efficiency.

\subsection{Stochastic Approximation}
Stochastic approximation methods \cite{robbins1951stochastic, nevel1976stochastic, SA92} are a family of iterative methods primarily used to solve root-finding or optimization problems. These methods address functions of the form $f(\theta) = \mathbb{E}_{\bm{z}}[F(\theta, \bm{z})]$, which represent the expected value of a function \( F \) that depends on a random variable \( z \). The goal is to infer properties of \( f \) without directly evaluating it. We focus primarily on its applications involving gradient estimation. 

\paragraph{Forward Gradient.}\label{sec:fg}
FG methods \cite{Wengert64, SilverGDHH22, baydin22, RenKLH23} update model parameters using directional gradients along multiple random perturbation directions and belong to the family of stochastic approximation techniques.
More formally, given a differentiable function \(f: \mathbb{R}^N \rightarrow \mathbb{R}\), the gradient at a given input \(\theta \in \mathbb{R}^N\) can be approximated as
\begin{equation}\label{eq:fg}
\hat{\nabla} f(\theta) := \left(  \nabla f(\theta)^{\top}z \right) z.
\end{equation}
There are multiple choices for the random variables $z$. All distributions satisfying \(\mathbb{E}[z] = 0\) and \(\mathbb{E}[zz^\top] = I_N\) are qualified, such as the standard Gaussian distribution and the Rademacher distribution. This way, for any given \(\theta\), \(\hat{\nabla} f(\theta)\) is also an unbiased estimator of \(\nabla f(\theta)\) since
\begin{equation}
\begin{aligned}
\mathbb{E} [\hat{\nabla} f(\theta)] &= \mathbb{E} \left[ \left(  \nabla f(\theta)^{\top}z \right) z \right] = \mathbb{E} \left[ z z^\top \right] \nabla f(\theta) \\&= I_N \nabla f(\theta)  = \nabla f(\theta).
\end{aligned}
\end{equation}
In practice, to reduce variance, Monte Carlo gradient estimation can be performed by averaging forward gradients over multiple random directions \cite{baydin22,hu2023}.
Utilizing forward-mode automatic differentiation techniques \cite{WilliamsZ89, Pearlmutter94}, the Jacobian-vector product \(\nabla_\theta  f(\theta)^{\top} z\) can be computed efficiently with a single forward pass. This enables forward gradient learning \cite{Wengert64, SilverGDHH22, baydin22, RenKLH23}, which updates model parameters based on the directional gradient along a random perturbation direction, enabling backpropagation-free training.

\end{document}